\documentclass[10pt,letterpaper]{article}
\usepackage[top=1in,bottom=1in,left=1in,right=1in]{geometry}

\usepackage[utf8]{inputenc}
\usepackage{hyperref}
\usepackage{url}
\usepackage{booktabs}
\usepackage{amsfonts}
\usepackage{nicefrac}
\usepackage{xcolor}

\usepackage{amsmath, amssymb}
\usepackage{amsthm}
\usepackage{graphicx}
\usepackage[capitalize,noabbrev]{cleveref}

\newtheorem{theorem}{Theorem}

\newtheorem{lemma}{Lemma}
\newtheorem{definition}{Definition}
\newtheorem{remark}{Remark}

\usepackage{natbib}
\usepackage{caption}
\usepackage{authblk}
\usepackage{parskip}

\title{Support Recovery in Sparse PCA with Non-Random Missing Data}

\author[1]{Hanbyul Lee}
\author[1]{Qifan Song}
\author[2]{Jean Honorio}
\affil[1]{Department of Statistics, Purdue University}
\affil[2]{Department of Computer Science, Purdue University}

\date{\today}

\begin{document}

\maketitle

\begin{abstract}
  
We analyze a practical algorithm for sparse PCA on incomplete and noisy data under a general non-random sampling scheme.
The algorithm is based on a semidefinite relaxation of the $\ell_1$-regularized PCA problem.
We provide theoretical justification that under certain conditions, 
we can recover the support of the sparse leading eigenvector with high probability by obtaining a unique solution. 
The conditions involve the spectral gap between the largest and second-largest eigenvalues of the true data matrix, the magnitude of the noise, and the structural properties of the observed entries.
The concepts of algebraic connectivity and irregularity are used to describe the structural properties of the observed entries.
We empirically justify our theorem with synthetic and real data analysis.
We also show that our algorithm outperforms several other sparse PCA approaches especially when the observed entries have good structural properties.
As a by-product of our analysis, we provide two theorems to handle a deterministic sampling scheme, which can be applied to other matrix-related problems.   
  
\end{abstract}

\section{Introduction}
\label{sec:introduction}

When principal components possess a certain sparsity structure, standard principal component analysis (PCA) is not preferred due to poor interpretability and the inconsistency of solutions under high-dimensional settings \cite{paul2007asymptotics, nadler2008finite, johnstone2009consistency}.
To solve these issues, {\it sparse PCA} has been proposed, which enforces sparsity in the PCA solution so that dimension reduction and variable selection can be simultaneously performed. 
Theoretical and algorithmic research on sparse PCA has been actively conducted over the past few years \cite{zou2006sparse, amini2008high, journee2010generalized, berk2019certifiably, 
richtarik2021alternating}.

In this paper, we focus on the case that the data to which sparse PCA is applied, are not completely observed, but partially missing. 
Missing data frequently occurs in a wide range of machine learning problems, and sparse PCA is no exception. 
This has led to several works that offer reliable solutions to sparse PCA on missing data \cite{lounici2013sparse, kundu2015approximating, park2019sparse, lee2022support}.
However, these methods make a restrictive assumption that the entries are observed uniformly at random,
and analyze only the effect of the observation rate on the solvability of the problem.

In real world applications, data missing patterns are usually not at random.
Furthermore, when the study focuses on the case of uniformly random missing, one fails to recognize other important observation-pattern-specific factors (e.g., some topological properties of the realized observation pattern) that may affect the solvability of the problem or the performance of an algorithm.
Accordingly, the analysis under a deterministic missing scheme without any assumptions on its pattern is in demand.
With this motivation, there have been several works in other matrix-related problems which handle a deterministic and general sampling scheme \cite{lee2013matrix, heiman2014deterministic, bhojanapalli2014universal},
while sparse PCA has not received such attention.

Examining a general sampling scheme is particularly important for sparse PCA
because structural properties of the observed entries, not just observation rate,
do affect the solvability of the problem.
Imagine that the leading eigenvector $\pmb{u}$ of a symmetric matrix is sparse and denote its support by $J$.
The goal of our sparse PCA problem is to exactly recover the support $J$ from an incomplete data matrix.
Note that non-zero values in $\pmb{u}$ only affect the entries in the $J\times J$ sub-matrix.
Therefore, to recover the support, we need to observe a sufficient number of entries in the $J\times J$ sub-matrix.
In fact, if we do not observe any entry in a specific row of the $J\times J$ sub-matrix, we can never identify the corresponding index as an element of the support, and thus exact recovery fails.
This implies that we need to observe the entries in the $J\times J$ sub-matrix abundantly and evenly;
in other words, if we think of a graph having the observed entries as its edge set (we call this an observation graph), its $J\times J$ sub-graph needs to be well-connected and have similar node degrees.

With this intuition, 
we aim to explore a sparse PCA algorithm for incomplete data under a general non-random sampling scheme,
which would be successful when observed entries have good structural properties.
The algorithm we consider is a semidefinite relaxation of the $l_1$ regularized PCA (we call this the SDP algorithm).
This algorithm has been analyzed on complete data and has been shown to have theoretically good properties \cite{d2004direct, lei2015sparsistency}.
It has also been shown to work well for incomplete data, when the observation rate is sufficiently large under the uniform random sampling scheme \cite{lee2022support}.

Our main contribution is as follows:
we provide theoretical justification (i.e., \cref{thm:sufficient_conditions}) that we can exactly recover the true support $J$ with high probability by obtaining a unique solution of the SDP algorithm with incomplete and noisy observation, under proper conditions. 
The conditions involve the spectral gap between the largest and second-largest eigenvalues of the true matrix, the magnitude of noise, and especially, structural properties of the observation graph.
We mathematically articulate the structural graph properties by using two interpretable concepts: algebraic connectivity and irregularity
(see \cref{def:algebraic_connectivity} and \cref{def:irregularity}).
Through these graph properties, we demonstrate that the algorithm works well if the $J\times J$ sub-graph is well-connected and has similar node degrees.
We note that the graph properties apply to any type of undirected graph,
that is, our theorem is applicable to any general deterministic sampling scheme.
To the best of our knowledge, this is the first work on sparse PCA with incomplete data under the general non-random sampling assumption.

We empirically validate our theorem with synthetic and real data analysis in \cref{sec:experiments}.
In particular, our simulation results show that the performance of the SDP algorithm is solely determined by the properties we derive in our theorem, 
which is a strong justification of our theory.
In real data analysis, we show that the SDP algorithm outperforms several other sparse PCA approaches, and only the SDP algorithm benefits from the good structure of the observation graph.

As by-products of our analysis, we provide two theorems to handle the non-random sampling scheme: 
the tail bound for random matrices with independent sub-Gaussian values in a subset of entries, and the bound of the difference between complete and incomplete matrices under non-random missingness (see \cref{subsec:by_products}).
These theorems can be applied to other matrix-related problems, when matrix entries are missing with a non-random sampling scheme.

\textbf{Paper Organization.}
In \cref{sec:problem_definition}, we introduce the support recovery problem in sparse PCA with incomplete and noisy observation. 
We discuss several applications of the problem here.
In \cref{sec:methods}, we present the SDP algorithm, and define several graph properties used in the theorem.
We present our main theorem in \cref{sec:main_results}, along with by-products of our analysis and the method to choose a tuning parameter in the SDP algorithm.
In \cref{sec:experiments}, we illustrate results from our empirical study.

\textbf{Notation.} 
Matrices are bold capital (e.g., $\pmb{A}$), vectors are bold lowercase (e.g., $\pmb{a}$), and scalars or entries are not bold.
$A_{i,j}$ and $a_i$ represent the $(i,j)$-th and $i$-th entries of $\pmb{A}$ and $\pmb{a}$, respectively.
For any index sets $I$ and $J$, $\pmb{A}_{I,J}$ and $\pmb{a}_I$ denote
the $|I|\times|J|$-dimensional sub-matrix of $\pmb{A}$ consisting of rows in $I$ and columns in $J$, and
the $|I|$-dimensional sub-vector of $\pmb{a}$ consisting of the entries in $J$,
respectively.
For any positive integer $n$, we denote $[n]:=\{1,\dots,n\}$.

$\|\pmb{a}\|_1$ and $\|\pmb{a}\|_2$ represent the $l_1$ and $l_2$ norms of $\pmb{a}$.
$\|\pmb{A}\|_2$ and $\|\pmb{A}\|_*$ indicate the spectral and nuclear norms of $\pmb{A}$.
We let $\|\pmb{A}\|_{1,1} = \sum_{i}\sum_{j}|A_{i,j}|$ and $\|\pmb{A}\|_{\max} = \max_{i,j}|A_{i,j}|$.
$\lambda_i(\pmb{A})$ represents the $i$-th largest eigenvalue of $\pmb{A}$.
The trace of $\pmb{A}$ is denoted by $tr(\pmb{A})$, and the matrix inner product of $\pmb{A}$ and $\pmb{B}$ is denoted by $\langle \pmb{A}, \pmb{B} \rangle$.
$\pmb{A}\circ\pmb{B}$ represents the Hadamard product of $\pmb{A}$ and $\pmb{B}$.

$f(x) = O(g(x))$ means that there exists a positive constant $C$ such that $f(x)\leq C g(x)$ asymptotically.
$f(x) = \Omega(g(x))$ is equivalent to $g(x) = O(f(x))$.
$f(x) = \tilde{O}(g(x))$ is shorthand for $f(x) = O(g(x)\log^k x)$ for some $k$.

\section{Problem Definition}
\label{sec:problem_definition}

\textbf{Sparse Principal Component.}
Let $\pmb{M}^*\in\mathbb{R}^{d\times d}$ be an unknown symmetric matrix
and 
$\pmb{M}^* = \sum_{k\in [d]}\lambda_k(\pmb{M}^*) \pmb{u}_k \pmb{u}_k^\top$
be the spectral decomposition of $\pmb{M}^*$,
where $\lambda_1(\pmb{M}^*) > \lambda_2(\pmb{M}^*) \geq \cdots \geq \lambda_d(\pmb{M}^*)$ are its eigenvalues and $\pmb{u}_1,\dots,\pmb{u}_d \in \mathbb{R}^d$ are the corresponding eigenvectors.
For identifiability of the leading eigenvector, we consider that $\lambda_1(\pmb{M}^*)$ is strictly greater than $\lambda_2(\pmb{M}^*)$.
We assume that the leading eigenvector $\pmb{u}_1$ of $\pmb{M}^*$ is sparse, i.e., for some index set $J\in [d]$,
$$
\begin{cases}
u_{1,i} \neq 0 & \text{if } i\in J, \\
u_{1,i} = 0 & \text{otherwise.}
\end{cases}
$$	
With a notation $supp(\pmb{a}):=\{i\in[d] : a_{i}\neq 0\}$ for any vector $\pmb{a}\in\mathbb{R}^d$, we can write $J = supp(\pmb{u}_1)$.
Also, we denote the size of $J$ by $s$. That is, $s = |J|$.

\textbf{Incomplete and Noisy Observation.}
Suppose that we have only noisy observations of the entries of $\pmb{M}^*$ over a fixed sampling set $\Omega \subseteq [d]\times[d]$.
Specifically, we observe a symmetric matrix $\pmb{M}\in\mathbb{R}^{d\times d}$ such that 
$$
M_{i,j}
=
\begin{cases}
M^*_{i,j} + N_{i,j}  & \text{if } (i,j)\in \Omega, \\
0 & \text{otherwise}
\end{cases}
$$
for $i,j \in [d]$, where $N_{i,j}$ is the noise at location $(i,j)$.
We assume that $N_{i,j}$'s are symmetric about zero and follow a sub-Gaussian distribution independently, i.e., 
$\mathbb{E}e^{\theta N_{i,j}}\leq e^{\frac{\sigma^2 \theta^2}{2}}$ for any $\theta \geq 0$ and some $\sigma \geq 0$.

\textbf{Goal.}
In this paper, we aim to exactly recover the true support $J$ of the leading eigenvector $\pmb{u}_1$ of $\pmb{M}^*$ only from the incomplete and noisy observation $\pmb{M}$.

\textbf{Applications.}
Here, we point out several motivating applications of the problem.
\begin{itemize}
\item 
\textit{Single-cell RNA sequence (scRNA-seq) data analysis.}
In scRNA-seq data, the cells are divided into several distinct types which can be characterized with only a small number of genes among tens of thousands of genes \cite{park2019sparse}.
Sparse PCA can be effectively utilized here to reduce the dimension (from numerous cells to a few cell types) and to select a small number of genes that affect the reduced data. 
However, scRNA-seq data usually have many missing values due to technical and sampling issues, and a sparse PCA method designed for incomplete data needs to be applied.
\item 
\textit{Covariance analysis in vertical federated setting.}
In vertical federated setting, 
one aims to utilize local data sets with different features about the same set of subjects to train machine learning models \cite{yang2019federated}.
Consider the case that we want to analyze the covariance matrix of the whole data.
Due to communication restriction or privacy protection requirement, only the sample covariance matrix of a few number of features can be provided to the analyst.
In this case, the analyst can only have an incomplete (noise-injected) sample covariance matrix.
Our analysis is applied to the case that sparse PCA is conducted in such setting.
\item 
\textit{Anomaly detection in network.}
Sparse PCA is useful for anomaly detection in network data sets \cite{singh2011anomalous}.
Consider an edge-weighted network data, where the weight for each edge can represent various types of information, e.g., similarity of two nodes, traffic volume between connected locations, or social distance between individuals or groups.
By applying sparse PCA to the weight matrix of the network, we can find an anomalous collection of individuals or nodes.
However, since unconnected nodes cause missingness, analysis for incomplete data must be considered.
\end{itemize}

\section{Methods}
\label{sec:methods}

In this section, we introduce the algorithm used to solve the sparse PCA problem,
and define several graph properties which will be utilized in our main theorem.

\textbf{SDP Algorithm.}
For the support recovery of the leading eigenvector, an intuitive approach is imposing a regularization term on the PCA quadratic loss.
When using the $l_1$ regularizer, the optimization problem can be written as:
\begin{equation*}
\hat{\pmb{x}} = \underset{\|\pmb{x}\|_2 = 1}{\arg\max}~ \pmb{x}^\top \pmb{M}\pmb{x} - \rho \|\pmb{x}\|_1^2.
\end{equation*}
Here, the true support $J$ is estimated with $supp(\hat{\pmb{x}})$.
However, the objective is non-convex and difficult to solve.
Therefore, we suggest the following semidefinite relaxation as an alternative:
\begin{equation}
\label{eq:sdp_problem}
\hat{\pmb{X}} = \underset{\pmb{X}\succeq 0 \text{ and } tr(\pmb{X}) = 1}{\arg\max}~ 
\langle \pmb{M}, \pmb{X} \rangle - \rho \|\pmb{X}\|_{1,1},
\end{equation}
where we estimate $J$ by $\hat{J} = supp(diag(\hat{\pmb{X}}))$.
We call this the SDP algorithm.
Efficient scalable SDP solvers exist \cite{yurtsever2021scalable}, so the SDP algorithm is computationally friendly.
This approach has been shown to have good theoretical properties and work well in practice for both of complete and randomly missing data \cite{d2004direct, lei2015sparsistency, lee2022support},
so we focus on the SDP algorithm in this paper.

\begin{remark}
We note that in the implementation of the SDP algorithm, we use the matrix $\pmb{M}$ where zero is imputed in the missing entries, without applying any matrix completion or imputation methods.
This is because matrix completion can introduce unwanted bias under inappropriate conditions.
It is well-known that most of the matrix completion methods can be successful only under the low-rank assumption.
In this paper, however, we allow the true matrix $\pmb{M}^*$ to be not necessarily low-rank.
In \cref{subsec:real_data_analysis}, we provide experimental evidence  showing that the SDP algorithm with zero-imputed $\pmb{M}$ performs well when the observation has a good structural property,
while the result yielded from matrix completion does not achieve good performance overall.
\end{remark}

\textbf{Graph Properties.}
Before presenting our main theorem in the next section,
we define several graph terminologies and properties which are involved in the theorem.

We first introduce the observation graph $\mathcal{G} = (\mathcal{V},\mathcal{E})$,
which is an undirected graph associated with the fixed sampling set $\Omega$,
that is, $\mathcal{V} = [d]$ and $(i,j)\in \mathcal{E}$ if and only if $(i,j)\in \Omega$.
Note that $\mathcal{G}$ is allowed to contain loops.
We denote the adjacency matrix corresponding to $\mathcal{G}$ by $\pmb{A}_\mathcal{G}$.
With this notation, we can write $\pmb{M} =\pmb{A}_\mathcal{G}\circ(\pmb{M}^* + \pmb{N})$ where $\pmb{N}$ is the noise matrix whose $(i,j)$-th entry is $N_{i,j}$.

Below are several convenient notations about graphs.
\begin{itemize}
\item $\mathcal{G}_{J,J}$, $\mathcal{G}_{J,J^c}$, $\mathcal{G}_{J^c,J^c}$: 
For the observation graph $\mathcal{G}$,
we denote by $\mathcal{G}_{J,J}$, $\mathcal{G}_{J,J^c}$ and $\mathcal{G}_{J^c,J^c}$ the sub-graphs of $\mathcal{G}$ which consist of only the edges inside $J\times J$, $J\times J^c$, and $J^c\times J^c$, respectively. 
$\mathcal{G}_{J,J}$ and $\mathcal{G}_{J^c,J^c}$ are undirected graphs with vertex sets $J$ and $J^c$, respectively.
$\mathcal{G}_{J,J^c}$ is a bipartite graph with independent vertex sets $J$ and $J^c$.
\item $\overline{\mathcal{G}}$: 
For any graph $\mathcal{G}$, $\overline{\mathcal{G}}$ denotes a graph which has the same vertex set as $\mathcal{G}$, but whose edge set is the complement of that of $\mathcal{G}$.
That is, the edge sets of $\mathcal{G}$ and $\overline{\mathcal{G}}$ are disjoint and their union forms a complete graph.
\item $\Delta_{\max}(\mathcal{G})$, $\Delta_{\min}(\mathcal{G})$:
For any graph $\mathcal{G}$, $\Delta_{\max}(\mathcal{G})$ and $\Delta_{\min}(\mathcal{G})$ denote the maximum and minimum node degrees of $\mathcal{G}$, respectively.
\end{itemize}

Now, we define two important structural graph properties, {\it algebraic connectivity} and {\it irregularity}.
Both properties are crucial to explain the effect of the structure of the observation graph on the solvability of our support recovery problem.

Algebraic connectivity, the well-known concept to measure the graph connectivity, is defined as follows:
\begin{definition}[Algebraic Connectivity]
\label{def:algebraic_connectivity}
The algebraic connectivity of a graph $\mathcal{G}$, denoted by $\phi(\mathcal{G})$, is the second-smallest eigenvalue of the Laplacian matrix of $\mathcal{G}$.
The magnitude of $\phi(\mathcal{G})$ reflects how well connected the overall graph is.
\end{definition}

Next, we define the concept of irregularity, which is an uncommon graphical quantity that proves crucial in presenting our results.

\begin{definition}[Irregularity]
\label{def:irregularity}
For any undirected graph $\mathcal{G}$ such that $\Delta_{\max}(\mathcal{G}) \geq \phi(\mathcal{G})$ and $\Delta_{\max}(\overline{\mathcal{G}}) \geq \phi(\overline{\mathcal{G}})$,
the irregularity of $\mathcal{G}$ is defined as
$$
\psi(\mathcal{G}) := \max\big\{\Delta_{\max}(\mathcal{G}) - \phi(\mathcal{G}), \Delta_{\max}(\overline{\mathcal{G}}) - \phi(\overline{\mathcal{G}}) \big\}.$$
The magnitude of $\psi(\mathcal{G})$ reflects how different the node degrees of $\mathcal{G}$ are.
\end{definition}

To better interpret the above concept, we derive some lower and upper bound results below.
Let $\pmb{A}_\mathcal{G}$ be the adjacency matrix of $\mathcal{G}$.
Then we can derive the following inequality\footnotemark:
\begin{align*}
\max_{\|\pmb{x}\|_2 = 1, \pmb{x}\perp \pmb{1}} \pmb{x}^\top \pmb{A}_\mathcal{G} \pmb{x}
\leq 
\Delta_{\max}(\mathcal{G}) - \phi(\mathcal{G})
\leq 
\max_{\|\pmb{x}\|_2 = 1, \pmb{x}\perp \pmb{1}} \pmb{x}^\top \pmb{A}_\mathcal{G} \pmb{x}
+ \Delta_{\max}(\mathcal{G}) - \Delta_{\min}(\mathcal{G})
\end{align*}
where $\pmb{1} = (1, 1, \dots, 1)^\top$.

\begin{itemize}
\item $\max_{\|\pmb{x}\|_2 = 1, \pmb{x}\perp \pmb{1}} \pmb{x}^\top \pmb{A}_\mathcal{G} \pmb{x}$: Among different $\pmb{A}_\mathcal{G}$'s having the same largest and second-largest eigenvalues, we can see that the one corresponding to a regular graph (a graph is regular when each node has the same degree) has the smallest magnitude of $\underset{\|\pmb{x}\|_2 = 1, \pmb{x}\perp \pmb{1}}{\max} \pmb{x}^\top \pmb{A}_\mathcal{G} \pmb{x}$.
This is because a regular graph has a normalized vector of $\pmb{1}$ as its leading eigenvector.
\item $\Delta_{\max}(\mathcal{G}) - \Delta_{\min}(\mathcal{G})$: This quantity decreases as nodes of $\mathcal{G}$ have similar degrees.
\end{itemize}
Since $\overline{\mathcal{G}}$ is regular when $\mathcal{G}$ is regular,
and $\Delta_{\max}(\mathcal{G}) - \Delta_{\min}(\mathcal{G})
= \Delta_{\max}(\overline{\mathcal{G}}) - \Delta_{\min}(\overline{\mathcal{G}})$,
the same argument holds for $\Delta_{\max}(\overline{\mathcal{G}}) - \phi(\overline{\mathcal{G}})$.
Therefore, we can say that as $\mathcal{G}$ is closer to a regular graph, the value of $\psi(\mathcal{G})$ decreases.
This is the reason why we name this concept the ``irregularity".

\footnotetext{
The inequality is derived as follows:
let $\pmb{D}_\mathcal{G}$ be a diagonal matrix whose diagonal entries are the node degrees of $\pmb{A}_\mathcal{G}$. Then,
$
\Delta_{\max}(\mathcal{G}) - \phi(\mathcal{G}) 
= 
\underset{\|\pmb{x}\|_2 = 1}{\max} \pmb{x}^\top \pmb{D}_\mathcal{G} \pmb{x}
- \underset{\|\pmb{x}\|_2 = 1, \pmb{x}\perp \pmb{1}}{\min} \pmb{x}^\top (\pmb{D}_\mathcal{G}-\pmb{A}_\mathcal{G}) \pmb{x}
\geq
\underset{\|\pmb{x}\|_2 = 1, \pmb{x}\perp \pmb{1}}{\max} \pmb{x}^\top \pmb{A}_\mathcal{G} \pmb{x}
$.
Also,
$
\Delta_{\max}(\mathcal{G}) - \phi(\mathcal{G}) 
\leq
\underset{\|\pmb{x}\|_2 = 1}{\max} \pmb{x}^\top \pmb{D}_\mathcal{G} \pmb{x}
+\underset{\|\pmb{x}\|_2 = 1, \pmb{x}\perp \pmb{1}}{\max} \pmb{x}^\top \pmb{A}_\mathcal{G} \pmb{x}
-\underset{\|\pmb{x}\|_2 = 1, \pmb{x}\perp \pmb{1}}{\min} \pmb{x}^\top \pmb{D}_\mathcal{G} \pmb{x}
\leq
\underset{\|\pmb{x}\|_2 = 1, \pmb{x}\perp \pmb{1}}{\max} \pmb{x}^\top \pmb{A}_\mathcal{G} \pmb{x}
+ \Delta_{\max}(\mathcal{G}) - \Delta_{\min}(\mathcal{G}).
$
}

\begin{remark}
Note that $\Delta_{\max}(\mathcal{G}) \geq \phi(\mathcal{G})$ and $\Delta_{\max}(\overline{\mathcal{G}}) \geq \phi(\overline{\mathcal{G}})$, i.e., $\psi(\mathcal{G})\geq 0$ holds in most cases except for some extreme types of graphs, e.g., complete graph without loops.
\end{remark}

\section{Main Results}
\label{sec:main_results}

Now, we introduce our main theorem, which shows the sufficient condition for the SDP algorithm to exactly recover the true support $J$.

\begin{theorem}
\label{thm:sufficient_conditions}
Under the problem definition in \cref{sec:problem_definition},
assume that the following condition holds: with some constant $c>0$,
\begin{align}
&
\|\pmb{M}^*_{J,J}\|_2 \cdot \psi(\mathcal{G}_{J,J})
+ \sigma \sqrt{\Delta_{\max}(\mathcal{G}_{J,J})\log s}
+ s \| \pmb{M}^*_{J^c,J} \|_{2}
+ \frac{1}{\sqrt{s}}\| \pmb{M}^*_{J^c,J^c}\|_2
\nonumber \\&
+ \sigma s\sqrt{ \max\big\{ \Delta_{\max}(\mathcal{G}_{J,J^c}), \Delta_{\max}(\mathcal{G}_{J^c,J^c}) \big\} \log d}
\leq \frac{c\phi(\mathcal{G}_{J,J}) \bar{\lambda}(\pmb{M}^*) \cdot \min_{i \in J}|u_{1,i}|}{s}
\label{eq:suff_cond}
\end{align}
where $\bar{\lambda}(\pmb{M}^*) := \lambda_1(\pmb{M}^*)-\lambda_2(\pmb{M}^*)$ and $\psi(\mathcal{G}_{J,J}) \geq 0$.
If the tuning parameter $\rho$ is properly chosen,
the optimal solution $\hat{\pmb{X}}$ to the optimization problem (\ref{eq:sdp_problem}) is unique
and satisfies $supp(diag(\hat{\pmb{X}})) = J$
with probability at least $1-O(s^{-1})$.
\end{theorem}

In a nutshell, \cref{thm:sufficient_conditions} asserts that the SDP algorithm produces reliable solutions under certain conditions imposed on 
the spectral gap $\bar{\lambda}(\pmb{M}^*)$, the noise intensity parameter $\sigma$, the matrix norms,
and 
the graph properties of the sub-graphs $\mathcal{G}_{J,J}$, $\mathcal{G}_{J,J^c}$ and $\mathcal{G}_{J^c,J^c}$.

To better understand the condition, we consider the setting that $\|\pmb{M}^*_{J,J}\|_2 = O(\bar{\lambda}(\pmb{M}^*))$ and 
$\min_{i \in J}|u_{1,i}| = \Omega(s^{-\frac{1}{2}})$, for instance.
We note that the first inequality holds as long as
$\|\pmb{M}^*_{J,J}\|_2 = \lambda_1(\pmb{M}^*_{J,J})$ and
$\frac{\lambda_2(\pmb{M}^*)}{\lambda_1(\pmb{M}^*)}\leq c$ for some $c\in (0,1)$,
and the second inequality holds when all the non-zero entries of the sparse leading eigenvector are of the same level of magnitude.
In this case, we can rewrite \eqref{eq:suff_cond} as follows:
\begin{align*}
&\frac{\psi(\mathcal{G}_{J,J})}{\phi(\mathcal{G}_{J,J})}
= {O}\Big(\frac{1}{s\sqrt{s}} \Big),
\\&
\sigma 
=  
\tilde{O} \Big( \frac{\phi(\mathcal{G}_{J,J})}{\sqrt{\max\{ \Delta_{\max}(\mathcal{G}_{J,J^c}), \Delta_{\max}(\mathcal{G}_{J^c,J^c}) \}}}\cdot \frac{\bar{\lambda}(\pmb{M}^*)}{s^2 \sqrt{ s}} \Big),
\\&
\| \pmb{M}^*_{J^c,J} \|_{2} 
=  
{O} \Big( \frac{\phi(\mathcal{G}_{J,J}) \bar{\lambda}(\pmb{M}^*)}{s^2\sqrt{s}}  \Big),
\\&
\| \pmb{M}^*_{J^c,J^c}\|_2 
=  
{O} \Big( \frac{\phi(\mathcal{G}_{J,J}) \bar{\lambda}(\pmb{M}^*)}{s} \Big).
\end{align*}

The first condition about the structural graph properties of $\mathcal{G}_{J,J}$ states that
for the algorithm to be successful, 
the sub-graph $\mathcal{G}_{J,J}$ is desired to have sufficiently large connectivity $\phi(\mathcal{G}_{J,J})$ and small irregularity $\psi(\mathcal{G}_{J,J})$.
This implies that we need to observe the entries in the $J\times J$ sub-matrix of the true matrix abundantly and evenly.
This result fits well with our first intuition discussed in the introduction.

The second through fourth conditions mean that the noise and the norms of $\pmb{M}^*_{J^c,J}$ and $\pmb{M}^*_{J^c,J^c}$ need to be well-controlled for the success of the algorithm.
This is in accordance with our common sense, since large $\sigma$ or $\pmb{M}^*$ values outside $J\times J$ matrix will mask the true information.
These conditions are alleviated when the connectivity of $\mathcal{G}_{J,J}$ is large (especially in the second condition, larger than the number of observed entries outside $J\times J$) and when the spectral gap $\bar{\lambda}(\pmb{M}^*)$ is large.

It is worth mentioning that a sufficiently large spectral gap requirement is to ensure the uniqueness and identifiability of the projection matrix with respect to the principal subspace, which has been also discussed in \citet{lei2015sparsistency} and \citet{lee2022support}.

\textbf{Proof Technique.}
Detailed proof of \cref{thm:sufficient_conditions} is given in \cref{sec:proof_of_thm1}.
At a high level, we use the KKT conditions under the primal-dual witness framework,
and find the sufficient conditions which guarantee that the solution of \eqref{eq:sdp_problem} is unique and satisfies $supp(diag(\hat{\pmb{X}})) = J$ with high probability.
We apply several techniques in the derivation, including Davis Kahan sin$\Theta$ theorem and Weyl's inequality for principal subspace estimation.

One challenge in the proof is to handle a deterministic sampling scheme.
We note that traditional concentration inequalities, such as matrix Bernstein inequality, are not useful to derive tail bounds under non-random missingness.
To overcome this, we obtain and utilize two theorems: one is for the tail bound of the matrix whose subset of entries are random, and the other is for bounding the difference between complete and incomplete matrices under a deterministic sampling scheme.
These two results can be widely used in other matrix-related problems for non-random missing data as well.
We introduce these two theorems as by-products of our analysis in the next section.

\subsection{By-products}
\label{subsec:by_products}

Below is a tail bound for the matrix containing independent sub-Gaussian random values only in a fixed subset of entries. 

\begin{theorem}[Tail Bound for Random Matrix with Independent Sub-Gaussian Values in a Subset of Entries]
\label{thm:by_product_tail_bound}
Consider a random matrix $\pmb{Z}\in \mathbb{R}^{m\times n}$ whose subset of entries independently follow sub-Gaussian distributions which are symmetric about zero and have parameter $\sigma > 0$, while the other entries are fixed as zero.
That is, there exists an index set $S \subseteq \{(i,j)~|~ i \in [m], j \in [n]\}$ such that for $i \in [m]$ and $j \in [n]$,
$$
Z_{i,j} =  
\begin{cases}
N_{i,j} & \text{if } (i,j)\in S, \\
0 & \text{otherwise},
\end{cases}
$$
where each $N_{i,j}$ is symmetric about zero and satisfies $\mathbb{E}e^{\theta N_{i,j}}\leq e^{\frac{\sigma^2 \theta^2}{2}}$ for any $\theta >0$.
Then for any $t \geq 0$,
$$
\mathbb{P} [\| \pmb{Z} \|_2 \geq t ]
\leq 
2(m+n)\cdot \exp\Big(-\frac{t^2}{2\sigma^2 \Delta_{\max}(\mathcal{G}_S)}\Big),
$$
where $\mathcal{G}_S$ is a bipartite graph whose vertex and edge sets are $[m]\times[n]$ and $S$, respectively.
\end{theorem}

We defer the proof to \cref{sec:proof_of_tail_bound}.
In the derivation of \cref{thm:sufficient_conditions}, we use the above theorem to obtain the tail bounds of $\|(\pmb{A}_\mathcal{G})_{J,J}\circ \pmb{N}_{J,J}\|_2$, $\|(\pmb{A}_\mathcal{G})_{J,J^c}\circ \pmb{N}_{J,J^c}\|_2$ and $\|(\pmb{A}_\mathcal{G})_{J^c,J^c}\circ \pmb{N}_{J^c,J^c}\|_2$.

Next is the bound of the difference between complete and incomplete matrices under non-random missingness.
We note that this is an extended result of Theorem 4.1 in \citet{bhojanapalli2014universal}.
While the theorem of \citet{bhojanapalli2014universal} is limited to the case that the observation graph is regular,
our theorem applies to any general undirected observation graph.
For regular graphs, our bound coincides with that of \citet{bhojanapalli2014universal}, i.e., our result generalizes the result of \citet{bhojanapalli2014universal}.

\begin{theorem}
\label{thm:by_product_diff_bound}
Consider a symmetric matrix $\pmb{Y}$ with dimension $n$.
Let $\pmb{Y} = \sum_{k\in [r]} \lambda_k(\pmb{Y}) \pmb{v}_k \pmb{v}_k^\top$ be the spectral decomposition of $\pmb{Y}$, where $r$ is rank of $\pmb{Y}$.
Define $\tau := \max_{i\in[n]} \sum_{k\in [r]} {v}^2_{k,i}$.

Also, consider an undirected graph $\mathcal{G}$ with $n$ nodes and denote its adjacency matrix by $\pmb{A}_\mathcal{G}$.
Then,
$$
\|\pmb{Y} - \frac{n}{\phi(\mathcal{G})}\pmb{A}_\mathcal{G}\circ \pmb{Y} \|_2 \leq \frac{n\tau\psi(\mathcal{G})}{\phi(\mathcal{G})}\cdot  \|\pmb{Y}\|_2.
$$
\end{theorem}

The proof is given in \cref{sec:proof_of_diff_bound}.
We use the above theorem to bound $\| \pmb{M}^*_{J,J} - \frac{s}{\phi(\mathcal{G}_{J,J})} (\pmb{A}_\mathcal{G})_{J,J}\circ \pmb{M}^*_{J,J} \|_2$ in the proof of \cref{thm:sufficient_conditions}.

\subsection{Choice of the Tuning Parameter}
\label{subsec:choice_of_tuning_parameter}

The theoretical choice of $\rho$ \eqref{eq:theoretical_choice_of_rho} is useless in practice since it relies on unknown quantities. 
Therefore, certain tuning procedure over $\rho$ is necessary for the implementation of the SDP algorithm \eqref{eq:sdp_problem}.
Our suggestion is to find $\rho$ to maximize the following AIC type criterion (see also \citet{qi2013sparse}):
$$
C_\rho = (1-a)\frac{\langle \pmb{M}, \hat{\pmb{X}}_\rho \rangle}{\langle \pmb{M}, \hat{\pmb{X}}_0 \rangle}
+ a\Big(1-\frac{|supp(diag(\hat{\pmb{X}}_\rho))|}{d}\Big).
$$
Here, $\hat{\pmb{X}}_\rho$ and $\hat{\pmb{X}}_0$ refer to the solutions of the SDP algorithm where the tuning parameters are set to be $\rho$ and $0$, respectively.
$\langle \pmb{M}, \hat{\pmb{X}}_\rho \rangle$ and $\langle \pmb{M}, \hat{\pmb{X}}_0 \rangle$ represent the explained variances of $\pmb{M}$ by the solutions.

The first term of the criterion is a measure for the quality of the estimate, and the second term penalizes for the complexity of the solution.
$a\in (0,1)$ is the weight to be chosen by practitioners. 
As one needs a sparse principal component, a relatively large value of $a$ is suggested.
In the experiments, we find that $0.4\leq a \leq 0.6$ generally work well.

\section{Experiments}
\label{sec:experiments}

In this section, we present several empirical results on synthetic and real data sets which support our theoretical results.

\subsection{Simulations}
\label{subsec:simulations}

The goal of this simulation study is to demonstrate the effects of the structural properties of the observation graph, the spectral gap between the largest and second-largest eigenvalues of the true matrix, and the magnitude of the noise, on the success of the support recovery by the SDP algorithm.

In particular, we want to check whether the performance of the SDP algorithm is solely determined by the properties we derive in \cref{thm:sufficient_conditions}.
We utilize the following rescaled parameter for this:
\begin{equation}
\label{eq:rescaled}
Rescaled = \frac{\text{LHS of \eqref{eq:suff_cond}}}{\text{RHS of \eqref{eq:suff_cond} without constant } c}.
\end{equation}
If the performance of the algorithm versus this rescaled parameter is the same across different settings, then we can empirically justify that 
the performance is solely determined by the factors in the rescaled parameter.
This kind of approach has been used in \citet{wainwright2009sharp} for sparse linear regression.

\textbf{Setting.}
We use synthetic data sets generated in the following manner.
The orthonormal eigenvectors of $\pmb{M}^*$ are randomly selected, 
while the leading eigenvector $\pmb{u}_1$ is made to be sparse and have $s$ non-zero entries with the values of $\frac{1}{\sqrt{s}}$.
$\lambda_2(\pmb{M}^*), \dots, \lambda_d(\pmb{M}^*)$ are randomly selected from a normal distribution with mean $0$ and standard deviation $1$,
and $\lambda_1(\pmb{M}^*)$ is set to $\lambda_2(\pmb{M}^*)$ plus the spectral gap.
We set the matrix dimension $d$ to be $50$ and the support size $s$ to be $10$.

We generate the observation graph $\mathcal{G}$ to have $1250$ edges out of $2500$.
The value of $\frac{\psi(\mathcal{G}_{J,J})}{\phi(\mathcal{G}_{J,J})}$ is set to be included in one of the ranges $0$ to $2$, $2$ to $4$, ..., or $16$ to $18$.
The entry-wise noise is randomly selected from a normal distribution with mean $0$ and standard deviation $\sigma$.

In each setting, we run the algorithm \eqref{eq:sdp_problem} and examine if the solution exactly recovers the true support $J$.
The tuning parameter $\rho \in \{0.025, 0.5, \dots, 1\}$ is selected by the method in \cref{subsec:choice_of_tuning_parameter} with $a=0.5$.
We repeat each experiment $100$ times with different random seeds, and calculate the rate of exact recovery in each setting.

\begin{figure}[t]
\vskip 0.2in
\begin{center}
\centerline{\includegraphics[width=0.7\columnwidth]{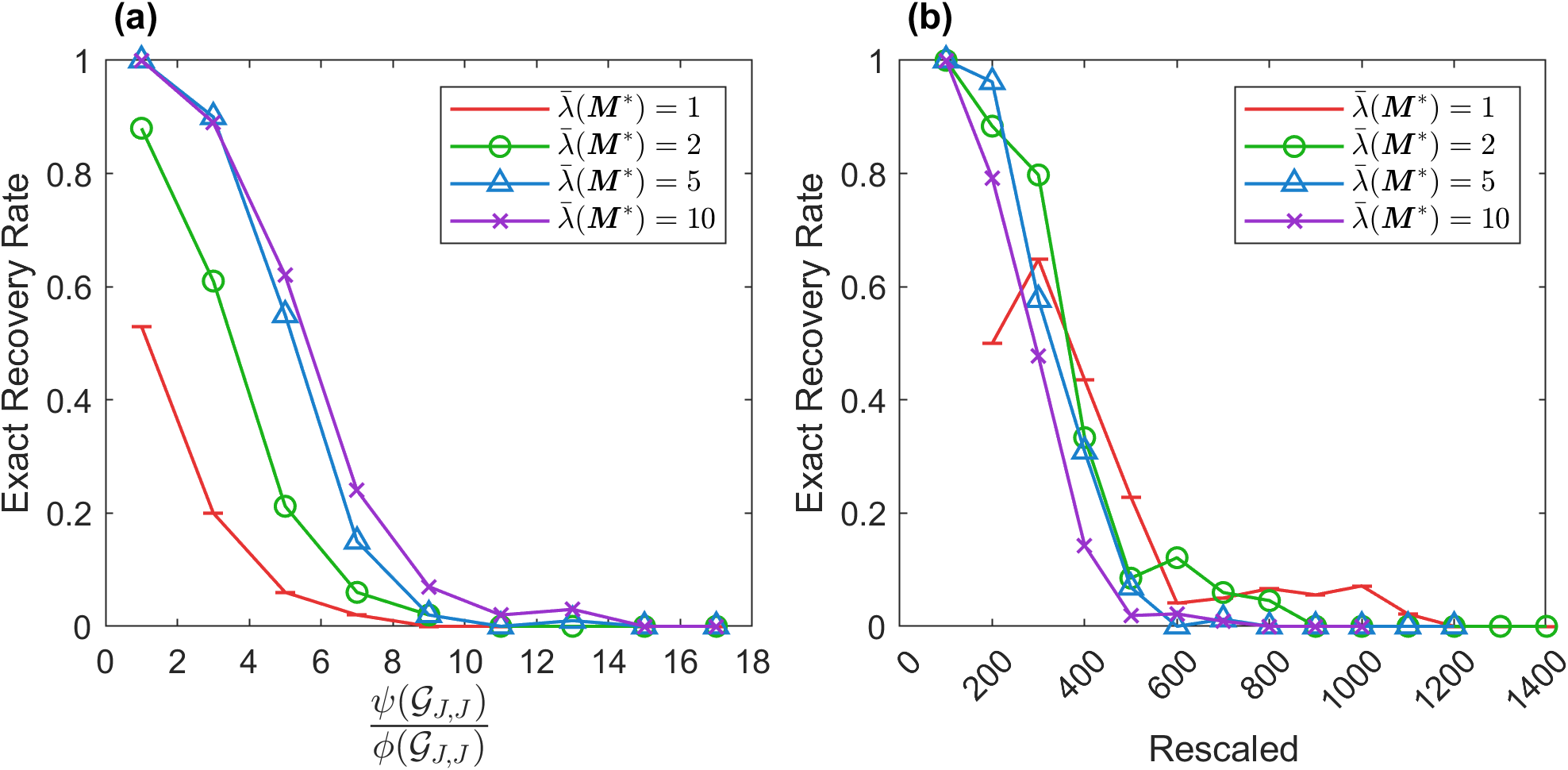}}
\caption{(a) Rate of exact recovery of support $J$ versus the graph parameter $\frac{\psi(\mathcal{G}_{J,J})}{\phi(\mathcal{G}_{J,J})}$ for four different spectral gaps $\bar{\lambda}(\pmb{M}^*)$,
(b) Same simulation results with exact recovery rate plotted versus the rescaled parameter in \eqref{eq:rescaled}.}
\label{fig:simulation1}
\end{center}
\vskip -0.2in
\end{figure}

\textbf{Results.}
\cref{fig:simulation1} shows the experimental results where we fix the noise parameter $\sigma$ as $0$ (noiseless) and try different spectral gaps $\bar{\lambda}(\pmb{M}^*) \in \{1, 2, 5, 10\}$ to check the effect of the spectral gap.
From (a) in \cref{fig:simulation1}, we can observe that the exact recovery rate increases as the spectral gap increases and the value of $\frac{\psi(\mathcal{G}_{J,J})}{\phi(\mathcal{G}_{J,J})}$ decreases, which is consistent with our theoretical findings.

\begin{figure}[t]
\vskip 0.2in
\begin{center}
\centerline{\includegraphics[width=0.7\columnwidth]{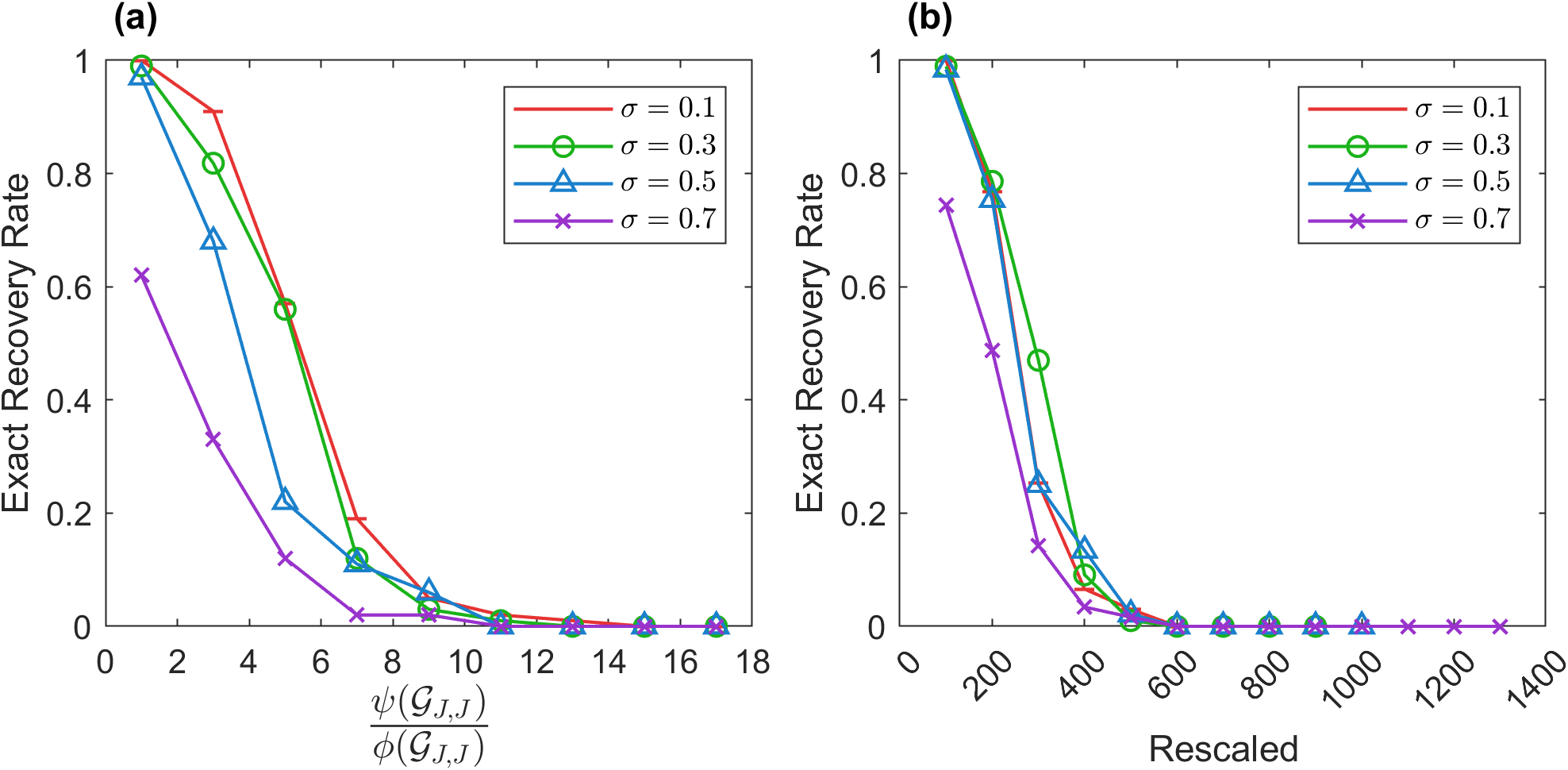}}
\caption{(a) Rate of exact recovery of support $J$ versus the graph parameter $\frac{\psi(\mathcal{G}_{J,J})}{\phi(\mathcal{G}_{J,J})}$ for four different noise parameters $\sigma$,
(b) Same simulation results with exact recovery rate plotted versus the rescaled parameter in \eqref{eq:rescaled}.}
\label{fig:simulation2}
\end{center}
\vskip -0.2in
\end{figure}

\cref{fig:simulation2} shows the experimental results where we fix the spectral gap $\bar{\lambda}(\pmb{M}^*)$ as $20$ and try different noise parameters $\sigma \in \{0.1, 0.3, 0.5, 0.7\}$ to check the effect of the magnitude of noise.
From (a) in \cref{fig:simulation2}, we can observe that the exact recovery rate increases as the standard deviation of the noise decreases and the value of $\frac{\psi(\mathcal{G}_{J,J})}{\phi(\mathcal{G}_{J,J})}$ decreases, which also supports our theorem.

In \cref{fig:simulation1} (b) and \cref{fig:simulation2} (b),
we can see that the curves of the exact recovery rate versus the rescaled parameter share almost the same pattern under different settings of $\bar{\lambda}(\pmb{M}^*)$ and $\sigma$.
This provides empirical justification of our theorem in the sense that the performance of the SDP algorithm is solely determined by the properties we derive in \cref{thm:sufficient_conditions}.

\subsection{Real Data Analysis}
\label{subsec:real_data_analysis}

The primary goal of this experimental study is to check if our SDP algorithm performs well compared to other sparse PCA algorithms on incomplete data.
We will also conduct a study to validate the selection criterion of the tuning parameter $\rho$ in \cref{subsec:choice_of_tuning_parameter}.

\textbf{Pitprops Data.}
The pitprops data, which stores 180 observations of 13 variables, has been a standard benchmark to evaluate algorithms for sparse PCA (see, e.g., 
\citet{zou2006sparse, shen2008sparse, journee2010generalized, qi2013sparse}).
We aim to compute a leading principal component of this data.
It has been revealed that on the complete data, a sparse solution with $6$ nonzero entries (with respect to the variables `topdiam', `length', `ringbut', `bowmax', `bowdist', `whorls') has a comparable explained variance with that of the dense solution from original PCA.
We focus on recovering these $6$ nonzero entries with incomplete data, by generating observation graphs synthetically.

\textbf{Setting.}
We impose missingness and noise on the complete $13\times 13$ covariance matrix in the following manner.
We generate the observation graph $\mathcal{G}$ to have $100$ edges out of $169$.
The value of $\frac{\psi(\mathcal{G}_{J,J})}{\phi(\mathcal{G}_{J,J})}$ is set to be included in one of the ranges $0$ to $0.2$, $0.2$ to $0.4$, ..., or $2$ to $2.2$.
The entry-wise noise is randomly selected from a normal distribution with mean $0$ and standard deviation $\sigma=0.1$.

We compare our SDP algorithm with three different methods.
Firstly, we consider two popular sparse PCA algorithms: the diagonal thresholding sparse PCA (DTSPCA) by \citet{johnstone2009consistency} and
the iterative thresholding sparse PCA (ITSPCA) by \citet{ma2013sparse},
which are efficient in computation and have theoretically good properties on complete data.
When implementing these methods on incomplete data, we treat missing cells as zero.
Secondly, we consider the combination of matrix completion and the SDP algorithm.
We first estimate the missing entries of the incomplete matrix $\pmb{M}$ by using the following matrix completion method based on the nuclear norm minimization:
$$
\tilde{\pmb{M}} = \underset{\pmb{Y}= \pmb{Y}^\top,~ \pmb{A}_\mathcal{G}\circ \pmb{Y} = \pmb{A}_\mathcal{G}\circ \pmb{M}}{\arg\min} \| \pmb{Y}\|_*.
$$
Then we implemented the SDP algorithm with the completed matrix $\tilde{\pmb{M}}$.

We run each algorithm $100$ times with different random seeds in each setting, and calculate the rate of exact recovery of $J$.
In the SDP algorithm, the tuning parameter $\rho \in \{0.025, 0.5, \dots, 1\}$ is selected by the method in \cref{subsec:choice_of_tuning_parameter} with $a=0.4$.
For the DTSPCA and ITSPCA algorithms, there is no well-known method to choose tuning parameters, so we try multiple values of tuning parameters and choose one with the largest exact recovery rate.

\begin{figure}[t]
\vskip 0.2in
\begin{center}
\centerline{\includegraphics[width=0.7\columnwidth]{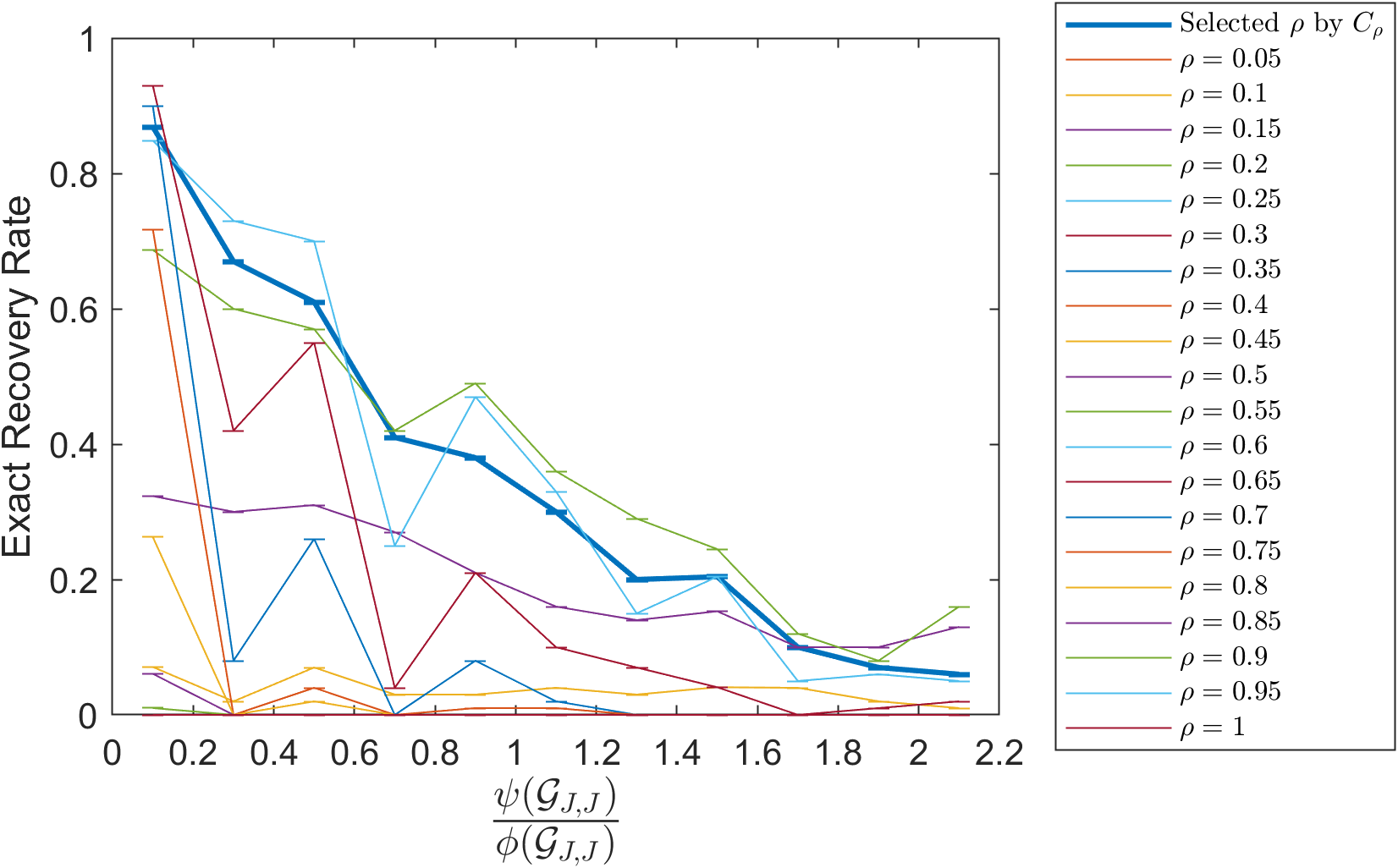}}
\caption{Rate of exact recovery of support $J$ versus the graph parameter $\frac{\psi(\mathcal{G}_{J,J})}{\phi(\mathcal{G}_{J,J})}$ for SDP algorithm with different values of the tuning parameter $\rho$.
The thick blue line indicates the result of the case that $\rho$ is selected by the criterion $C_\rho$ in \cref{subsec:choice_of_tuning_parameter}.}
\label{fig:tuning}
\end{center}
\vskip -0.2in
\end{figure}

\textbf{Results.}
In \cref{fig:tuning}, we can see that our method of selecting $\rho$ works well.
Here, we compare the result from our tuning method with those in the settings where the tuning parameter $\rho$ is fixed as a value among $\{0.05, 0.1, \dots, 1\}$ through all the repetitions.
We can see that the exact recovery rate from our tuning method is larger than most of the results where $\rho$ is fixed as one value.

\begin{figure}[t]
\vskip 0.2in
\begin{center}
\centerline{\includegraphics[width=0.5\columnwidth]{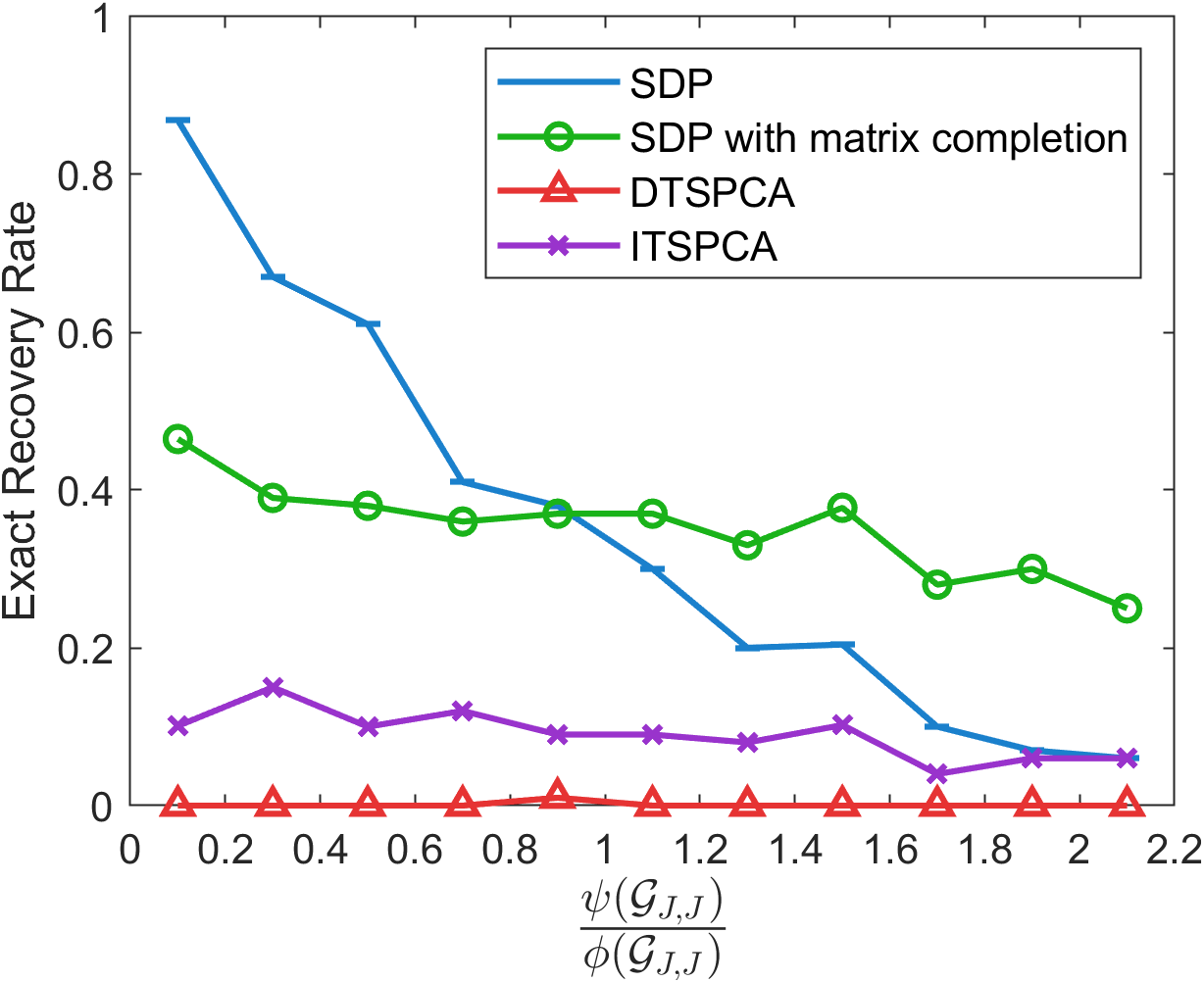}}
\caption{Rate of exact recovery of support $J$ versus the graph parameter $\frac{\psi(\mathcal{G}_{J,J})}{\phi(\mathcal{G}_{J,J})}$ for four different sparse PCA algorithms.}
\label{fig:comparison}
\end{center}
\vskip -0.2in
\end{figure}

\cref{fig:comparison} shows that our SDP method outperforms the other sparse PCA methods when the observation graph has a good structural property.
First, we observe that DTSPCA and ITSPCA algorithms do not work well at all (the exact recovery rates are always less than 0.2), even when the graph parameter $\frac{\psi(\mathcal{G}_{J,J})}{\phi(\mathcal{G}_{J,J})}$ is small enough.
Unlike our SDP method, those two algorithms can not benefit from good structure of the observation graph.

In addition, the SDP method with matrix completion has the exact recovery rate of around $0.4$ overall,
while our algorithm with zero-imputation produces the exact recovery rates greater than $0.6$ when the value of $\frac{\psi(\mathcal{G}_{J,J})}{\phi(\mathcal{G}_{J,J})}$ is small enough.
About the failure of the matrix completion approach, we conjecture the following rationale:
while the matrix completion algorithm can be successful under the low-rank assumption, the pitprops data is not low-rank,
which yields an unsuccessful matrix completion result.
Accordingly, the imputed cells introduce more noise into the inference and the result of sparse PCA becomes even worse than that of simply using zero for the missing entries.
We note that unlike matrix completion, our algorithm does not require the low-rank assumption to be successful.
Therefore, our algorithm has good performance under a general condition on the true matrix.

\section{Concluding Remarks}
\label{sec:concluding_remarks}

This paper examines the support recovery problem in sparse PCA with incomplete and noisy data, under a general and deterministic sampling scheme.
The problem formulation spans various applications including gene expression data analysis, vertical federated learning, and anomaly detection in networks.
We suggest a practical algorithm based on a semidefinite relaxation of the $\ell_1$-regularized PCA problem,
and provide sufficient conditions where the algorithm can exactly recover the true support with high probability.
The conditions involve the spectral gap, the noise parameter, the matrix norms, and the structural properties of the observation graph.
We show that the algorithm works well if we observe the entries in the $J\times J$ sub-matrix abundantly and evenly.
We empirically justify our theorem with synthetic and real data analysis,
and show that our algorithm outperforms several other sparse PCA approaches especially when the observation graph has a good structural property.

Our work primarily focuses on the theoretical understanding of the problem and the algorithm,
but we briefly discuss a practical use of our results here.
Our results present that as we observe the entries well in the $J\times J$ sub-matrix, we can recover the support $J$.
Unfortunately, we do not know the true support $J$ in practice, so it is impossible to check if the observation structure satisfies the sufficient condition.
One possible alternative is to check if any sub-graphs of the observation graph have good algebraic connectivity and irregularity.
If there is a sub-graph with connectivity or irregularity that is extremely low or high, then we could conservatively suspect that the result from the algorithm can not be fully trusted.
Here, we will need prior knowledge about the size of the support.

\bibliography{ref}
\bibliographystyle{plainnat}

\newpage

\appendix

\section{Proof of Theorem \ref{thm:sufficient_conditions}}
\label{sec:proof_of_thm1}

\subsubsection*{Step 1: Deriving Sufficient Conditions From KKT Conditions (Primal-dual Witness Approach)}

With the primal variable $\pmb{X}\in\mathbb{R}^{d\times d}$ and the dual variables $\pmb{Z}\in\mathbb{R}^{d\times d}$, $\pmb{\Lambda}\in\mathbb{R}^{d\times d}$ and $\mu \in \mathbb{R}$,
the Lagrangian of the problem \eqref{eq:sdp_problem} is written as
$$
L(\pmb{X}, \pmb{Z}, \pmb{\Lambda}, \mu)
= 
- \langle \pmb{M}, \pmb{X} \rangle + \rho \langle \pmb{X}, \pmb{Z} \rangle
- \langle \pmb{\Lambda}, \pmb{X} \rangle  + \mu \cdot (tr(\pmb{X}) - 1)
$$
where $Z_{ij} \in \partial |X_{ij}|$ for each $i,j \in [d]$.
According to the standard KKT condition, we can derive that
$(\hat{\pmb{X}}, \hat{\pmb{Z}}, \hat{\pmb{\Lambda}}, \hat{\mu})$ is optimal
if and only if the followings hold:
\begin{itemize}
\item Primal feasibility: 
$\hat{\pmb{X}} \succeq 0$, $tr(\hat{\pmb{X}}) = 1$
\item Dual feasibility: 
$\hat{\pmb{\Lambda}} \succeq 0$,
$\hat{Z}_{ij} \in \partial |\hat{X}_{ij}|$ for each $i,j \in [d]$
\item Complementary slackness: 
$\langle \hat{\pmb{\Lambda}}, \hat{\pmb{X}} \rangle = 0$
($\Leftrightarrow \hat{\pmb{\Lambda}} \hat{\pmb{X}} = 0$ if $\hat{\pmb{X}} \succeq 0$ and $\hat{\pmb{\Lambda}} \succeq 0$)
\item Stationarity: 
$\hat{\pmb{\Lambda}} = - \pmb{M} + \rho  \hat{\pmb{Z}} + \hat\mu \cdot \pmb{I}$.
\end{itemize}
By substituting $\hat{\pmb{\Lambda}}$ with $- \pmb{M} + \rho  \hat{\pmb{Z}} + \hat\mu \cdot \pmb{I}$,
it can be shown that the above conditions are equivalent to
\begin{align*}
&\hat{\pmb{X}} \succeq 0, tr(\hat{\pmb{X}}) = 1
\\&
\pmb{M} - \rho  \hat{\pmb{Z}} \preceq \hat\mu \pmb{I}
\\&
\hat{Z}_{ij} \in \partial |\hat{X}_{ij}| ~~\text{ for each } i,j \in [d]
\\&
(\pmb{M} - \rho  \hat{\pmb{Z}})\hat{\pmb{X}} = \hat\mu \cdot \hat{\pmb{X}}.
\end{align*}

To use the primal-dual witness construction, we now consider the following restricted problem:

\begin{equation}
\label{fps_restricted_problem}
\underset{\pmb{X}\succeq 0, tr(\pmb{X}) = 1 \text{ and } supp(\pmb{X})\subseteq J\times J}{\max}~ 
\langle \pmb{M}, \pmb{X} \rangle - \rho \|\pmb{X}\|_{1,1}.
\end{equation}
Similarly to the above, we can derive that
$\hat{\pmb{X}} = 
\begin{pmatrix}
\hat{\pmb{X}}_{J,J} & 0 \\
0 & 0
\end{pmatrix}
$\footnotemark 
is optimal to the problem (\ref{fps_restricted_problem}) if and only if
\begin{align*}
&\hat{\pmb{X}}_{J,J} \succeq 0, tr(\hat{\pmb{X}}_{J,J}) = 1
\\&
\pmb{M}_{J,J} - \rho  \hat{\pmb{Z}}_{J,J} \preceq \hat\mu \pmb{I}
\\&
\hat{Z}_{ij} \in \partial |\hat{X}_{ij}| ~~\text{ for each } i,j \in J
\\&
(\pmb{M}_{J,J} - \rho  \hat{\pmb{Z}}_{J,J})\hat{\pmb{X}}_{J,J} = \hat\mu \cdot \hat{\pmb{X}}_{J,J}.
\end{align*}

\footnotetext{
For clarity of exposition, our abuse of notation seemingly assumes $J=[s]$ when we join vectors and matrices.
It should be clear that for $J\neq [s]$, one will need to properly interleave vector entries or matrix rows/columns.
}

Now, we want for the above solution $\hat{\pmb{X}} = \begin{pmatrix}
\hat{\pmb{X}}_{J,J} & 0 \\
0 & 0
\end{pmatrix}$
to satisfy the optimality conditions of the original problem (\ref{eq:sdp_problem}).
Furthermore, by assuming the strict dual feasibility, we want to guarantee $supp(diag(\hat{\pmb{X}})) \subseteq J$.
We can easily derive the sufficient conditions listed below:
\begin{align*}
&\hat{\pmb{X}}_{J,J} \succeq 0, tr(\hat{\pmb{X}}_{J,J}) = 1
\\&
\pmb{M}_{J,J} - \rho  \hat{\pmb{Z}}_{J,J} \preceq \hat\mu \pmb{I}
\\&
\pmb{M} - \rho  \hat{\pmb{Z}} \preceq \hat\mu \pmb{I}
\\&
\hat{Z}_{ij} \in \partial |\hat{X}_{ij}| ~~\text{ for each } (i,j) \in J\times J
\\&
\hat{Z}_{ij} \in (-1, 1) ~~\text{ for each } (i,j) \notin J\times J
\\&
(\pmb{M}_{J,J} - \rho  \hat{\pmb{Z}}_{J,J})\hat{\pmb{X}}_{J,J} = \hat\mu \cdot \hat{\pmb{X}}_{J,J}
\\&
(\pmb{M}_{J^c,J} - \rho  \hat{\pmb{Z}}_{J^c,J})\hat{\pmb{X}}_{J,J} = 0.
\end{align*}
If the above conditions hold, then 
$\hat{\pmb{X}} = \begin{pmatrix}
\hat{\pmb{X}}_{J,J} & 0 \\
0 & 0
\end{pmatrix}$ is optimal to the problem (\ref{eq:sdp_problem}) and satisfies $supp(diag(\hat{\pmb{X}})) \subseteq J$.

Now, consider $\hat{\pmb{x}}, \hat{\pmb{z}} \in \mathbb{R}^{s}$ such that
\begin{align}
&\hat{z}_i = \text{sign}(u_{1,i}) \text{~~~~for all~} i\in J,
\nonumber \\
& \hat{\pmb{x}} \text{~is the leading eigenvector of~} \pmb{M}_{J,J}-\rho \hat{\pmb{z}} \hat{\pmb{z}}^\top.
\label{eq:x_hat}
\end{align}
Let $\hat{\pmb{X}}_{J,J} = \hat{\pmb{x}} \hat{\pmb{x}}^\top$ and $\hat{\pmb{Z}}_{J,J} = \hat{\pmb{z}} \hat{\pmb{z}}^\top$.
Then if the following conditions hold:
\begin{align}
& \text{sign}(u_{1,i}) = \text{sign}(\hat{x}_{i}) \text{~for all~} i\in J
~~\text{ or }~~ \text{sign}(u_{1,i}) = -\text{sign}(\hat{x}_{i}) \text{~for all~} i\in J
\label{cond1}
\\
&
(\pmb{M}_{J^c,J} - \rho  \hat{\pmb{Z}}_{J^c,J})\hat{\pmb{x}} = 0
~~\text{ and }~~ \|\hat{\pmb{Z}}_{J^c,J}\|_{\max} < 1
\label{cond2}
\\
&
\lambda_1(\pmb{M}_{J,J}-\rho \hat{\pmb{z}} \hat{\pmb{z}}^\top) = \lambda_1(\pmb{M}-\rho \hat{\pmb{Z}})
~~\text{ and }~~ \|\hat{\pmb{Z}}_{J^c,J^c}\|_{\max} < 1,
\label{cond3}
\end{align}
the above sufficient conditions are satisfied, that is, $\hat{\pmb{X}} := \begin{pmatrix}
\hat{\pmb{x}} \hat{\pmb{x}}^\top & 0 \\
0 & 0
\end{pmatrix}$
is optimal to the problem \eqref{eq:sdp_problem}.
Also, $supp(diag(\hat{\pmb{X}})) = J$ holds since $\text{sign}(u_{1,i}) = \text{sign}(\hat{x}_{i}\text{ or } -\hat{x}_{i})  \neq 0 \text{~for all~} i\in J$.

For the uniqueness, we need an additional condition presented in the following lemma.

\begin{lemma}
For 
$\hat{\pmb{X}}$
and
$\hat{\pmb{Z}}$
constructed above,
if the following condition holds:
\begin{equation}
\lambda_1(\pmb{M}_{J,J}-\rho \hat{\pmb{z}} \hat{\pmb{z}}^\top )
> \lambda_2(\pmb{M}_{J,J}-\rho \hat{\pmb{z}} \hat{\pmb{z}}^\top)
\label{cond4}
\end{equation}
then the solution $\hat{\pmb{X}}$ is a unique optimal solution to the problem (\ref{eq:sdp_problem}).
\end{lemma}

\begin{proof}
According to the standard primal-dual witness construction, we only need to show that under the condition, $\hat{\pmb{X}}_{J,J} = \hat{\pmb{x}} \hat{\pmb{x}}^\top$ is a unique optimal solution to the restricted problem (\ref{fps_restricted_problem}).

Assume that there exists another optimal solution to the problem (\ref{fps_restricted_problem}), say $\tilde{\pmb{X}}_{J,J}$.
Also, denote its dual optimal solution by $\tilde{\pmb{Z}}_{J,J}$. 
Then, we can write
\begin{align*}
& \langle \pmb{M}_{J,J}, \hat{\pmb{X}}_{J,J} \rangle - \rho \|\hat{\pmb{X}}_{J,J}\|_{1,1}
=
\langle \pmb{M}_{J,J}-\rho \hat{\pmb{z}} \hat{\pmb{z}}^\top,  \hat{\pmb{x}} \hat{\pmb{x}}^\top \rangle
= \hat{\pmb{x}}^\top (\pmb{M}_{J,J}-\rho \hat{\pmb{z}} \hat{\pmb{z}}^\top) \hat{\pmb{x}}
\\ &
=\langle \pmb{M}_{J,J}, \tilde{\pmb{X}}_{J,J} \rangle - \rho \|\tilde{\pmb{X}}_{J,J}\|_{1,1}
= \langle \pmb{M}_{J,J}-\rho \tilde{\pmb{Z}}_{J,J}, \tilde{\pmb{X}}_{J,J} \rangle.
\end{align*}

Recall that $\hat{\pmb{x}}$ is the leading eigenvector of $\pmb{M}_{J,J}-\rho \hat{\pmb{z}} \hat{\pmb{z}}^\top$, that is, 
$\hat{\pmb{x}}^\top (\pmb{M}_{J,J}-\rho \hat{\pmb{z}} \hat{\pmb{z}}^\top) \hat{\pmb{x}} = \lambda_1(\pmb{M}_{J,J}-\rho \hat{\pmb{z}} \hat{\pmb{z}}^\top)$.
Now, we will show that 
$\langle \pmb{M}_{J,J}-\rho \hat{\pmb{z}} \hat{\pmb{z}}^\top, \tilde{\pmb{X}}_{J,J} \rangle < \lambda_1(\pmb{M}_{J,J}-\rho \hat{\pmb{z}} \hat{\pmb{z}}^\top)$
for any matrix $\tilde{\pmb{X}}_{J,J} \neq \hat{\pmb{x}} \hat{\pmb{x}}^\top$ such that $\tilde{\pmb{X}}_{J,J}\succeq 0$ and $tr(\tilde{\pmb{X}}_{J,J}) = 1$.
Let $\tilde{\pmb{X}}_{J,J} = \sum_{i\in J}\theta_i \pmb{v}_i \pmb{v}_i^\top$, which is the spectral decomposition of $\tilde{\pmb{X}}_{J,J}$.
We can derive that
\begin{align*}
\langle \pmb{M}_{J,J}-\rho \hat{\pmb{z}} \hat{\pmb{z}}^\top, \tilde{\pmb{X}}_{J,J} \rangle
&= 
\langle \pmb{M}_{J,J}-\rho \hat{\pmb{z}} \hat{\pmb{z}}^\top, \sum_{i\in J}\theta_i \pmb{v}_i \pmb{v}_i^\top \rangle
=
\sum_{i\in J} \theta_i \pmb{v}_i^\top (\pmb{M}_{J,J}-\rho \hat{\pmb{z}} \hat{\pmb{z}}^\top) \pmb{v}_i
\leq
\lambda_1(\pmb{M}_{J,J}-\rho \hat{\pmb{z}} \hat{\pmb{z}}^\top)
\end{align*}
where the last inequality holds since $\sum_{i\in J} \theta_i = tr(\tilde{\pmb{X}}_{J,J}) = 1$ and 
$\pmb{v}_i^\top (\pmb{M}_{J,J}-\rho \hat{\pmb{z}} \hat{\pmb{z}}^\top) \pmb{v}_i\leq \lambda_1(\pmb{M}_{J,J}-\rho \hat{\pmb{z}} \hat{\pmb{z}}^\top)$.
Here, the equality holds only if $\theta_1 = 1$, $\theta_i = 0$ for $i\neq 1$ and $\pmb{v}_1 = \hat{\pmb{x}}$, that is, 
$\tilde{\pmb{X}}_{J,J} = \hat{\pmb{x}} \hat{\pmb{x}}^\top$.
Therefore, 
$\langle \pmb{M}_{J,J}-\rho \hat{\pmb{z}} \hat{\pmb{z}}^\top, \tilde{\pmb{X}}_{J,J} \rangle < \lambda_1(\pmb{M}_{J,J}-\rho \hat{\pmb{z}} \hat{\pmb{z}}^\top)$
for any matrix $\tilde{\pmb{X}}_{J,J} \neq \hat{\pmb{x}} \hat{\pmb{x}}^\top$ such that $\tilde{\pmb{X}}_{J,J}\succeq 0$ and $tr(\tilde{\pmb{X}}_{J,J}) = 1$.

With this fact, we can derive that
\begin{align*}
\langle \pmb{M}_{J,J}, \hat{\pmb{X}}_{J,J} \rangle - \rho \|\hat{\pmb{X}}_{J,J}\|_{1,1}
&= 
\hat{\pmb{x}}^\top (\pmb{M}_{J,J}-\rho \hat{\pmb{z}} \hat{\pmb{z}}^\top) \hat{\pmb{x}}
= \lambda_1(\pmb{M}_{J,J}-\rho \hat{\pmb{z}} \hat{\pmb{z}}^\top)
\\&>
\langle \pmb{M}_{J,J}-\rho \hat{\pmb{z}} \hat{\pmb{z}}^\top, \tilde{\pmb{X}}_{J,J} \rangle
=
\langle \pmb{M}_{J,J}-\rho \tilde{\pmb{Z}}_{J,J}, \tilde{\pmb{X}}_{J,J} \rangle
+ \rho \langle \tilde{\pmb{Z}}_{J,J}-\hat{\pmb{z}} \hat{\pmb{z}}^\top, \tilde{\pmb{X}}_{J,J} \rangle
\\&=
\langle \pmb{M}_{J,J}, \tilde{\pmb{X}}_{J,J} \rangle - \rho \|\tilde{\pmb{X}}_{J,J}\|_{1,1} 
+ \rho \langle \tilde{\pmb{Z}}_{J,J}-\hat{\pmb{z}} \hat{\pmb{z}}^\top, \tilde{\pmb{X}}_{J,J} \rangle.
\end{align*}
Since $\langle \pmb{M}_{J,J}, \hat{\pmb{X}}_{J,J} \rangle - \rho \|\hat{\pmb{X}}_{J,J}\|_{1,1} = \langle \pmb{M}_{J,J}, \tilde{\pmb{X}}_{J,J} \rangle - \rho \|\tilde{\pmb{X}}_{J,J}\|_{1,1} $ by the assumption, 
the above inequality implies 
$\langle \tilde{\pmb{Z}}_{J,J}-\hat{\pmb{z}} \hat{\pmb{z}}^\top, \tilde{\pmb{X}}_{J,J} \rangle <0$, that is, 
$\langle \tilde{\pmb{Z}}_{J,J}, \tilde{\pmb{X}}_{J,J} \rangle
< 
\langle \hat{\pmb{z}} \hat{\pmb{z}}^\top, \tilde{\pmb{X}}_{J,J} \rangle$.
This contradicts the fact that 
$\langle \tilde{\pmb{Z}}_{J,J}, \tilde{\pmb{X}}_{J,J} \rangle = \sup_{\|\pmb{Z}_{J,J}\|_{\max} \leq 1} \langle \pmb{Z}_{J,J}, \tilde{\pmb{X}}_{J,J} \rangle$,
and thus the desired result holds.

\end{proof}

\newpage

\subsubsection*{Step 2: Deriving Sufficient Conditions for \eqref{cond1}-\eqref{cond4}}

\begin{lemma}[Sufficient Condition for \eqref{cond1}]
\label{lemma_suff_cond_sign}
If the following inequality holds:
\begin{equation*}
\|\pmb{M}^*_{J,J}\|_2 \cdot \psi(\mathcal{G}_{J,J})
+ 2\sigma \sqrt{\Delta_{\max}(\mathcal{G}_{J,J})\log s}
+ s\rho
\leq
\frac{\phi(\mathcal{G}_{J,J}) \bar{\lambda}(\pmb{M}^*_{J,J}) \cdot \min_{i \in J}|u_{1,i}|}{2\sqrt{2} s},
\end{equation*}
then the condition \eqref{cond1} holds, that is, $\text{sign}(u_{1,i}) = \text{sign}(\hat{x}_{i}) \text{~for all~} i\in J
\text{ or } \text{sign}(u_{1,i}) = -\text{sign}(\hat{x}_{i}) \text{~for all~} i\in J$,
with probability at least $1- 2s^{-1}$.
\end{lemma}

\begin{proof}

By applying the Davis-Kahan sin$\Theta$ theorem, we obtain
$$
\|\pmb{u}_1 - \hat{\pmb{x}}\|_2 ~~\text{or}~~ \|\pmb{u}_1 + \hat{\pmb{x}}\|_2 \leq \frac{2\sqrt{2}}{\bar{\lambda}(\pmb{M}^*_{J,J})} 
\cdot \| \pmb{M}^*_{J,J} - \frac{s}{\phi(\mathcal{G}_{J,J})}(\pmb{M}_{J,J} - \rho \hat{\pmb{z}} \hat{\pmb{z}}^\top ) \|_2.
$$
By the triangle inequality, \cref{lemma_tail_bound} and \cref{thm:by_product_diff_bound}, we can upper bound
\begin{align*}
\| \pmb{M}^*_{J,J} - \frac{s}{\phi(\mathcal{G}_{J,J})}(\pmb{M}_{J,J} - \rho \hat{\pmb{z}} \hat{\pmb{z}}^\top ) \|_2
&\leq
\| \pmb{M}^*_{J,J} - \frac{s}{\phi(\mathcal{G}_{J,J})} \mathbb{E}[\pmb{M}_{J,J}] \|_2
+
\frac{s}{\phi(\mathcal{G}_{J,J})} \| \mathbb{E}[\pmb{M}_{J,J}] - \pmb{M}_{J,J} \|_2
+ \frac{s^2 \rho}{\phi(\mathcal{G}_{J,J})}
\\&\leq
\| \pmb{M}^*_{J,J} - \frac{s}{\phi(\mathcal{G}_{J,J})} (\pmb{A}_\mathcal{G})_{J,J}\circ \pmb{M}^*_{J,J} \|_2
+
\frac{s}{\phi(\mathcal{G}_{J,J})} \cdot 2\sigma \sqrt{ \Delta_{\max}(\mathcal{G}_{J,J}) \log s}
+ \frac{s^2 \rho}{\phi(\mathcal{G}_{J,J})}
\\&\leq
\frac{s\psi(\mathcal{G}_{J,J})}{\phi(\mathcal{G}_{J,J})} \cdot \| \pmb{M}^*_{J,J} \|_2
+
\frac{s}{\phi(\mathcal{G}_{J,J})} \cdot 2\sigma \sqrt{ \Delta_{\max}(\mathcal{G}_{J,J}) \log s}
+ \frac{s^2 \rho}{\phi(\mathcal{G}_{J,J})}
\end{align*}
with probability at least $1-2s^{-1}$.

Now, we have that
\begin{align*}
\|\pmb{u}_1 - \hat{\pmb{x}}\|_2 ~~\text{or}~~ \|\pmb{u}_1 + \hat{\pmb{x}}\|_2 
\leq 
\frac{2\sqrt{2}}{\bar{\lambda}(\pmb{M}^*_{J,J})} 
\cdot 
\bigg\{
\frac{s\psi(\mathcal{G}_{J,J})}{\phi(\mathcal{G}_{J,J})} \cdot \| \pmb{M}^*_{J,J} \|_2
+
\frac{s}{\phi(\mathcal{G}_{J,J})} \cdot 2\sigma \sqrt{ \Delta_{\max}(\mathcal{G}_{J,J}) \log s}
+ \frac{s^2 \rho}{\phi(\mathcal{G}_{J,J})}
\bigg\}.
\end{align*}
By Lemma \ref{lemma_sign}, if 
\begin{align*}
\frac{2\sqrt{2}}{\bar{\lambda}(\pmb{M}^*_{J,J})} 
\cdot 
\bigg\{
\frac{s\psi(\mathcal{G}_{J,J})}{\phi(\mathcal{G}_{J,J})} \cdot \| \pmb{M}^*_{J,J} \|_2
+
\frac{s}{\phi(\mathcal{G}_{J,J})} \cdot 2\sigma \sqrt{ \Delta_{\max}(\mathcal{G}_{J,J}) \log s}
+ \frac{s^2 \rho}{\phi(\mathcal{G}_{J,J})}
\bigg\}
\leq
\min_{i \in J}|u_{1,i}|,
\end{align*}
that is, 
\begin{equation*}
\|\pmb{M}^*_{J,J}\|_2 \cdot \psi(\mathcal{G}_{J,J})
+ 2\sigma \sqrt{\Delta_{\max}(\mathcal{G}_{J,J})\log s}
+ s\rho
\leq
\frac{\phi(\mathcal{G}_{J,J}) \bar{\lambda}(\pmb{M}^*_{J,J}) \cdot \min_{i \in J}|u_{1,i}|}{2\sqrt{2} s},
\end{equation*}
then $sign(u_{1,i}) = sign(\hat{x}_i)$ for all $i\in J$ or $sign(u_{1,i}) = -sign(\hat{x}_i)$ for all $i\in J$ with probability at least $1- 2s^{-1}$.

\end{proof}

\newpage

\begin{lemma}[Sufficient Condition for \eqref{cond2}]
\label{lemma_suff_cond_z2}
Let $\hat{\pmb{Z}}_{J^c,J} 
= \frac{1}{\rho \|\hat{\pmb{x}}\|_1}\pmb{M}_{J^c, J} \hat{\pmb{x}} \hat{\pmb{z}}^\top$.
Then it satisfies $(\pmb{M}_{J^c,J} - \rho  \hat{\pmb{Z}}_{J^c,J})\hat{\pmb{x}} = 0$.
Also,
if the following inequality holds:
$$
2\sigma\sqrt{ \Delta_{\max}(\mathcal{G}_{J,J^c}) \log d } 
+ \| \pmb{M}^*_{J^c,J} \|_{\max}
< \rho,
$$
then $\|\hat{\pmb{Z}}_{J^c,J}\|_{\max} < 1$ with probability at least $1-2d^{-1}$.
\end{lemma}

\begin{proof}
First, we can derive the upper bound of $\|\hat{\pmb{Z}}_{J^c,J}\|_{\max}$ as follows:
\begin{align*}
\|\hat{\pmb{Z}}_{J^c,J}\|_{\max}
&= 
\frac{1}{\rho \|\hat{\pmb{x}}\|_1}\|\pmb{M}_{J^c, J} \hat{\pmb{x}} \hat{\pmb{z}}^\top\|_{\max}
= \frac{1}{\rho \|\hat{\pmb{x}}\|_1} \cdot \max_{i\in J^c} \bigg| \sum_{j\in J} M_{i,j} \hat{x}_j \bigg|
\\ &\leq
\frac{1}{\rho \|\hat{\pmb{x}}\|_1} \cdot 
\Big( \max_{i\in J^c} \max_{j\in J} |M_{i,j}| \Big)\cdot \sum_{j\in J} | \hat{x}_j |
= \frac{1}{\rho } \cdot \|\pmb{M}_{J^c,J}\|_{\max}
\\&=
\frac{1}{\rho} \cdot \|\pmb{M}_{J^c,J} - \mathbb{E}[\pmb{M}_{J^c,J}] + \mathbb{E}[\pmb{M}_{J^c,J}] \|_{\max}
\\&\leq
\frac{1}{\rho} \cdot \|\pmb{M}_{J^c,J} - \mathbb{E}[\pmb{M}_{J^c,J}] \|_{\max} + \frac{1}{\rho} \cdot \|\mathbb{E}[\pmb{M}_{J^c,J}] \|_{\max}
\\&\leq
\frac{1}{\rho}\cdot 2\sigma\sqrt{ \Delta_{\max}(\mathcal{G}_{J,J^c}) \log d } + \frac{1}{\rho} \cdot \| \pmb{M}^*_{J^c,J} \|_{\max}
\end{align*}
where the last inequality holds with probability at least $1-2d^{-1}$, by Lemma \ref{lemma_tail_bound}. Hence,
if the following inequality holds:
$$
2\sigma\sqrt{\Delta_{\max}(\mathcal{G}_{J,J^c})  \log d} + \| \pmb{M}^*_{J^c,J} \|_{\max}
< \rho,
$$
then $\|\hat{\pmb{Z}}_{J^c,J}\|_{\max} < 1$ with probability at least $1-2d^{-1}$.

\end{proof}

\newpage

\begin{lemma}[Sufficient Condition for \eqref{cond3},\eqref{cond4}]
\label{lemma_suff_cond_z3}
Let $\hat{\pmb{Z}}_{J^c,J^c} = \frac{1}{\rho} 
\Big( \pmb{M}_{J^c,J^c} - \mathbb{E}[\pmb{M}_{J^c,J^c}] \Big)$.
If the condition in Lemma \ref{lemma_suff_cond_sign} holds and
the following inequalities hold:
\begin{align*}
&(1+\xi)\cdot \bigg(2\sigma\sqrt{\Delta_{\max}(\mathcal{G}_{J,J^c})  \log d} + \| \pmb{M}^*_{J^c,J}\|_2\bigg)\cdot (1+\sqrt{s})
\leq
\frac{\phi(\mathcal{G}_{J,J})}{2s}\cdot \bar{\lambda}(\pmb{M}^*_{J,J})\cdot \bigg(1 - \frac{1}{\sqrt{2}}\min_{i \in J}|u_{1,i}|\bigg),
\\&
(1+\xi)\cdot \| \pmb{M}^*_{J^c,J^c}\|_2
\leq
\frac{\phi(\mathcal{G}_{J,J})}{2s}\cdot \bar{\lambda}(\pmb{M}^*_{J,J})\cdot \bigg(1 - \frac{1}{2\sqrt{2}}\min_{i \in J}|u_{1,i}|\bigg),
\\&
2\sigma\sqrt{\Delta_{\max}(\mathcal{G}_{J^c,J^c}) \log d} < \rho,
\end{align*}
then 
$\lambda_1(\pmb{M}_{J,J}-\rho \hat{\pmb{z}} \hat{\pmb{z}}^\top) = \lambda_1(\pmb{M}-\rho \hat{\pmb{Z}})$,
$\lambda_1(\pmb{M}_{J,J}-\rho \hat{\pmb{z}} \hat{\pmb{z}}^\top )
> \lambda_2(\pmb{M}_{J,J}-\rho \hat{\pmb{z}} \hat{\pmb{z}}^\top)$
and 
$\|\pmb{Z}_{J^c,J^c}\|_{\max} < 1$ with probability at least $1-2s^{-1}-4d^{-1}$.
Here, $\xi \geq 0$ is a constant satisfying
$\| (\pmb{A}_\mathcal{G})_{J^c,J} \circ \pmb{M}^*_{J^c,J}\|_2 \leq (1+\xi)\cdot \| \pmb{M}^*_{J^c,J}\|_2$ and
$\| (\pmb{A}_\mathcal{G})_{J^c,J^c} \circ \pmb{M}^*_{J^c,J^c}\|_2 \leq (1+\xi)\cdot \| \pmb{M}^*_{J^c,J^c}\|_2$.
\end{lemma}

\begin{proof}
Lemma \ref{lemma_eig_ineq} shows that 
if the following inequality holds:
$$
\underbrace{\| \pmb{M}_{J^c,J}-\rho \hat{\pmb{Z}}_{J^c,J} \|_2^2}_{=: a_1}
\leq 
\underbrace{\big\{ \lambda_1(\pmb{M}_{J,J}-\rho \hat{\pmb{z}} \hat{\pmb{z}}^\top) - \lambda_2(\pmb{M}_{J,J}-\rho \hat{\pmb{z}} \hat{\pmb{z}}^\top) \big\} 
}_{=: a_2}
\cdot \underbrace{\big\{\lambda_1(\pmb{M}_{J,J}-\rho \hat{\pmb{z}} \hat{\pmb{z}}^\top) 
- \lambda_1( \pmb{M}_{J^c,J^c}-\rho \hat{\pmb{Z}}_{J^c,J^c} ) \big\}}_{=: a_3},
$$
then
$\lambda_1(\pmb{M}_{J,J}-\rho \hat{\pmb{z}} \hat{\pmb{z}}^\top) = \lambda_1(\pmb{M}-\rho \hat{\pmb{Z}})$.

\paragraph{Upper Bound of $a_1$:}
\begin{align*}
\| \pmb{M}_{J^c,J}-\rho \hat{\pmb{Z}}_{J^c,J} \|_2
&=
\bigg\| \pmb{M}_{J^c,J}-\rho\cdot \frac{1}{\rho \|\hat{\pmb{x}}\|_1} \pmb{M}_{J^c, J}\hat{\pmb{x}} \hat{\pmb{z}}^\top \bigg\|_2
=
\bigg\| \pmb{M}_{J^c,J} \cdot \bigg(I-\frac{\hat{\pmb{x}} \hat{\pmb{z}}^\top}{\|\hat{\pmb{x}}\|_1}  \bigg) \bigg\|_2
\\ &\leq
\| \pmb{M}_{J^c,J} \|_2 \cdot \bigg\| I-\frac{\hat{\pmb{x}} \hat{\pmb{z}}^\top}{\|\hat{\pmb{x}}\|_1}  \bigg\|_2
\leq
\| \pmb{M}_{J^c,J} \|_2 \cdot \bigg(1+\frac{\|\hat{\pmb{x}}\|_2 \|\hat{\pmb{z}}\|_2}{\|\hat{\pmb{x}}\|_1}  \bigg)
\\ &\leq
\| \pmb{M}_{J^c,J} \|_2 \cdot (1+\sqrt{s})
\\&=
\| \pmb{M}_{J^c,J} - \mathbb{E}[\pmb{M}_{J^c,J}] + \mathbb{E}[\pmb{M}_{J^c,J}] \|_2 \cdot (1+\sqrt{s})
\\&\leq
\big( \| \pmb{M}_{J^c,J} - \mathbb{E}[\pmb{M}_{J^c,J}]\|_2 + \| (\pmb{A}_\mathcal{G})_{J^c,J} \circ \pmb{M}^*_{J^c,J}\|_2 \big) \cdot (1+\sqrt{s})
\\&\leq
\bigg(2\sigma\sqrt{\Delta_{\max}(\mathcal{G}_{J,J^c})  \log d} + (1+\xi)\cdot \| \pmb{M}^*_{J^c,J}\|_2\bigg)\cdot (1+\sqrt{s})
\\&\leq
(1+\xi)\cdot \bigg(2\sigma\sqrt{\Delta_{\max}(\mathcal{G}_{J,J^c})  \log d} + \| \pmb{M}^*_{J^c,J}\|_2\bigg)\cdot (1+\sqrt{s})
\end{align*}
where the penultimate inequality holds with probability at least $1-2d^{-1}$, by Lemma \ref{lemma_tail_bound}.

\paragraph{Lower Bound of $a_2$:}
By Weyl's inequality,
\begin{align*}
\lambda_1(\pmb{M}_{J,J}-\rho \hat{\pmb{z}} \hat{\pmb{z}}^\top) - \lambda_2(\pmb{M}_{J,J}-\rho \hat{\pmb{z}} \hat{\pmb{z}}^\top)
&\geq
\frac{\phi(\mathcal{G}_{J,J})}{s}\cdot \lambda_1(\pmb{M}^*_{J,J}) - \frac{\phi(\mathcal{G}_{J,J})}{s}\cdot \lambda_2(\pmb{M}^*_{J,J})
-
2\cdot\| \frac{\phi(\mathcal{G}_{J,J})}{s}\pmb{M}^*_{J,J} - \pmb{M}_{J,J} + \rho \hat{\pmb{z}} \hat{\pmb{z}}^\top \|_2
\\ &\geq
\frac{\phi(\mathcal{G}_{J,J})}{s}\cdot \bar{\lambda}(\pmb{M}^*_{J,J})
-
\frac{2\phi(\mathcal{G}_{J,J}) \bar{\lambda}(\pmb{M}^*_{J,J}) \cdot \min_{i \in J}|u_{1,i}|}{2\sqrt{2} s}
\\&=
\frac{\phi(\mathcal{G}_{J,J})}{s}\cdot \bar{\lambda}(\pmb{M}^*_{J,J})\cdot \bigg(1 - \frac{1}{\sqrt{2}}\min_{i \in J}|u_{1,i}|\bigg)
\end{align*}
where the second inequality holds with probability at least $1-2s^{-1}$, by Lemma \ref{lemma_suff_cond_sign}.

\paragraph{Lower Bound of $a_3$:}
Finally, in a similar way to the above, we have that
\begin{align*}
\lambda_1(\pmb{M}_{J,J}-\rho \hat{\pmb{z}} \hat{\pmb{z}}^\top)
&\geq
\frac{\phi(\mathcal{G}_{J,J})}{s}\cdot \lambda_1(\pmb{M}^*_{J,J})
-
\| \frac{\phi(\mathcal{G}_{J,J})}{s}\pmb{M}^*_{J,J} - \pmb{M}_{J,J} + \rho \hat{\pmb{z}} \hat{\pmb{z}}^\top \|_2
\\ &\geq
\frac{\phi(\mathcal{G}_{J,J})}{s}\cdot \bar{\lambda}(\pmb{M}^*_{J,J})
-
\frac{\phi(\mathcal{G}_{J,J}) \bar{\lambda}(\pmb{M}^*_{J,J}) \cdot \min_{i \in J}|u_{1,i}|}{2\sqrt{2} s}
\\&=
\frac{\phi(\mathcal{G}_{J,J})}{s}\cdot \bar{\lambda}(\pmb{M}^*_{J,J})\cdot \bigg(1 - \frac{1}{2\sqrt{2}}\min_{i \in J}|u_{1,i}|\bigg)
\end{align*}
with probability at least $1-2s^{-1}$.
Also, since $\hat{\pmb{Z}}_{J^c,J^c} = \frac{1}{\rho} 
\Big( \pmb{M}_{J^c,J^c} - \mathbb{E}[\pmb{M}_{J^c,J^c}] \Big)$,
$$
\lambda_1( \pmb{M}_{J^c,J^c}-\rho \hat{\pmb{Z}}_{J^c,J^c} )
= \lambda_1( \mathbb{E}[\pmb{M}_{J^c,J^c}] )
= \lambda_1( (\pmb{A}_\mathcal{G})_{J^c,J^c} \circ \pmb{M}^*_{J^c,J^c} )
\leq
(1+\xi)\cdot \| \pmb{M}^*_{J^c,J^c}\|_2.
$$
Hence, $a_3$ is lower-bounded by 
$\frac{\phi(\mathcal{G}_{J,J})}{s}\cdot \bar{\lambda}(\pmb{M}^*_{J,J})\cdot \bigg(1 - \frac{1}{2\sqrt{2}}\min_{i \in J}|u_{1,i}|\bigg)
-(1+\xi)\cdot \| \pmb{M}^*_{J^c,J^c}\|_2$.

By using the bounds of $a_1$, $a_2$ and $a_3$,
we can derive that if the following inequalities hold:
\begin{align*}
&(1+\xi)\cdot (2\sigma\sqrt{\Delta_{\max}(\mathcal{G}_{J,J^c})  \log d} + \| \pmb{M}^*_{J^c,J}\|_2)\cdot (1+\sqrt{s})
\leq
\frac{\phi(\mathcal{G}_{J,J})}{2s}\cdot \bar{\lambda}(\pmb{M}^*_{J,J})\cdot \bigg(1 - \frac{1}{\sqrt{2}}\min_{i \in J}|u_{1,i}|\bigg),
\\&
(1+\xi)\cdot \| \pmb{M}^*_{J^c,J^c}\|_2
\leq
\frac{\phi(\mathcal{G}_{J,J})}{2s}\cdot \bar{\lambda}(\pmb{M}^*_{J,J})\cdot \bigg(1 - \frac{1}{2\sqrt{2}}\min_{i \in J}|u_{1,i}|\bigg),
\end{align*}
then $\lambda_1(\pmb{M}_{J,J}-\rho \hat{\pmb{z}} \hat{\pmb{z}}^\top) = \lambda_1(\pmb{M}-\rho \hat{\pmb{Z}})$ with probability at least $1-2s^{-1} - 2d^{-1}$.

Lastly, by using the lower bound of $a_2$, we can derive that if $\bar{\lambda}(\pmb{M}^*_{J,J}) > 0$, then $\lambda_1(\pmb{M}_{J,J}-\rho \hat{\pmb{z}} \hat{\pmb{z}}^\top )
> \lambda_2(\pmb{M}_{J,J}-\rho \hat{\pmb{z}} \hat{\pmb{z}}^\top)$ holds with probability at least $1-2s^{-1}$.
Note that $\bar{\lambda}(\pmb{M}^*_{J,J}) > 0$ holds because $\bar{\lambda}(\pmb{M}^*_{J,J}) \geq \bar{\lambda}(\pmb{M}^*) > 0$ by our problem definition.
Also, by using Lemma \ref{lemma_tail_bound}, we can see that if $2\sigma\sqrt{\Delta_{\max}(\mathcal{G}_{J^c,J^c})  \log d} < \rho$,
then $\|\pmb{Z}_{J^c,J^c}\|_{\max} = \frac{1}{\rho} \| 
 \pmb{M}_{J^c,J^c} - \mathbb{E}[\pmb{M}_{J^c,J^c}] \|_{\max}
\leq \frac{2\sigma}{\rho} \sqrt{\Delta_{\max}(\mathcal{G}_{J^c,J^c})  \log d}
 < 1$ holds with probability at least $1-2d^{-1}$.
\end{proof}

\newpage

\subsubsection*{Step 3: Final Result}

By above lemmas, we can show the following theorem, which is the formal version of \cref{thm:sufficient_conditions} in the main text.

\begin{theorem}
Under the problem definition in \cref{sec:problem_definition},
assume that the following inequalities hold:
\begin{align*}
&\|\pmb{M}^*_{J,J}\|_2 \cdot \psi(\mathcal{G}_{J,J})
+ 2\sigma \sqrt{\Delta_{\max}(\mathcal{G}_{J,J})\log s}
+ s\rho
\leq
\frac{\phi(\mathcal{G}_{J,J}) \bar{\lambda}(\pmb{M}^*_{J,J}) \cdot \min_{i \in J}|u_{1,i}|}{2\sqrt{2} s},
\\&
2\sigma\sqrt{ \Delta_{\max}(\mathcal{G}_{J,J^c}) \log d } 
+ \| \pmb{M}^*_{J^c,J} \|_{\max}
< \rho,
\\&
(1+\xi)\cdot \bigg(2\sigma\sqrt{\Delta_{\max}(\mathcal{G}_{J,J^c})  \log d} + \| \pmb{M}^*_{J^c,J}\|_2\bigg)\cdot (1+\sqrt{s})
\leq
\frac{\phi(\mathcal{G}_{J,J})}{2s}\cdot \bar{\lambda}(\pmb{M}^*_{J,J})\cdot \bigg(1 - \frac{1}{\sqrt{2}}\min_{i \in J}|u_{1,i}|\bigg),
\\&
(1+\xi)\cdot \| \pmb{M}^*_{J^c,J^c}\|_2
\leq
\frac{\phi(\mathcal{G}_{J,J})}{2s}\cdot \bar{\lambda}(\pmb{M}^*_{J,J})\cdot \bigg(1 - \frac{1}{2\sqrt{2}}\min_{i \in J}|u_{1,i}|\bigg),
\\&
2\sigma\sqrt{\Delta_{\max}(\mathcal{G}_{J^c,J^c}) \log d} < \rho,
\end{align*}
where $\xi \geq 0$ is a constant satisfying
$\| (\pmb{A}_\mathcal{G})_{J^c,J} \circ \pmb{M}^*_{J^c,J}\|_2 \leq (1+\xi)\cdot \| \pmb{M}^*_{J^c,J}\|_2$ and
$\| (\pmb{A}_\mathcal{G})_{J^c,J^c} \circ \pmb{M}^*_{J^c,J^c}\|_2 \leq (1+\xi)\cdot \| \pmb{M}^*_{J^c,J^c}\|_2$.
Then $\hat{\pmb{X}} := \begin{pmatrix}
\hat{\pmb{x}} \hat{\pmb{x}}^\top & 0 \\
0 & 0
\end{pmatrix}$
with $\hat{\pmb{x}}$ defined in \eqref{eq:x_hat}
is a unique optimal solution to the problem \eqref{eq:sdp_problem}, and it satisfies $supp(diag(\hat{\pmb{X}})) = J$,
with probability at least $1-2s^{-1} - 4d^{-1}$.
\end{theorem}

Consider the following choice of the tuning parameter $\rho$:
\begin{equation}
\label{eq:theoretical_choice_of_rho}
\rho = 
2\sigma\sqrt{ \max\big\{ \Delta_{\max}(\mathcal{G}_{J,J^c}), \Delta_{\max}(\mathcal{G}_{J^c,J^c}) \big\} \log d}
+ \| \pmb{M}^*_{J^c,J} \|_{\max}.
\end{equation}
Then it suffices to satisfy
\begin{align*}
&\|\pmb{M}^*_{J,J}\|_2 \cdot \psi(\mathcal{G}_{J,J})
+ 2\sigma \sqrt{\Delta_{\max}(\mathcal{G}_{J,J})\log s}
+ 2\sigma s\sqrt{ \max\big\{ \Delta_{\max}(\mathcal{G}_{J,J^c}), \Delta_{\max}(\mathcal{G}_{J^c,J^c}) \big\} \log d}
+ s\| \pmb{M}^*_{J^c,J} \|_{\max}
\\&~~~~~~\leq
\frac{\phi(\mathcal{G}_{J,J}) \bar{\lambda}(\pmb{M}^*_{J,J}) \cdot \min_{i \in J}|u_{1,i}|}{2\sqrt{2} s},
\\&
(1+\xi)\cdot \bigg(2\sigma\sqrt{\Delta_{\max}(\mathcal{G}_{J,J^c})  \log d} + \| \pmb{M}^*_{J^c,J}\|_2\bigg)\cdot (1+\sqrt{s})
\leq
\frac{\phi(\mathcal{G}_{J,J})}{2s}\cdot \bar{\lambda}(\pmb{M}^*_{J,J})\cdot \bigg(1 - \frac{1}{\sqrt{2}}\min_{i \in J}|u_{1,i}|\bigg),
\\&
(1+\xi)\cdot \| \pmb{M}^*_{J^c,J^c}\|_2
\leq
\frac{\phi(\mathcal{G}_{J,J})}{2s}\cdot \bar{\lambda}(\pmb{M}^*_{J,J})\cdot \bigg(1 - \frac{1}{2\sqrt{2}}\min_{i \in J}|u_{1,i}|\bigg).
\end{align*}

Note that $\min_{i \in J}|u_{1,i}| \leq \frac{1}{\sqrt{s}}$. Hence, the second and third inequalities are satisfied when
\begin{align*}
&2\sigma\sqrt{\Delta_{\max}(\mathcal{G}_{J,J^c})  \log d} + \| \pmb{M}^*_{J^c,J}\|_2
\leq
\frac{c_1 \phi(\mathcal{G}_{J,J}) \bar{\lambda}(\pmb{M}^*_{J,J}) \min_{i \in J}|u_{1,i}|}{s},
\\&
\frac{1}{\sqrt{s}}\cdot \| \pmb{M}^*_{J^c,J^c}\|_2
\leq
\frac{c_2 \phi(\mathcal{G}_{J,J}) \bar{\lambda}(\pmb{M}^*_{J,J}) \min_{i \in J}|u_{1,i}|}{s}
\end{align*}
for some constants $c_1, c_2 > 0$.
Therefore, the sufficient conditions hold if
\begin{multline*}
\|\pmb{M}^*_{J,J}\|_2 \cdot \psi(\mathcal{G}_{J,J})
+ \sigma \sqrt{\Delta_{\max}(\mathcal{G}_{J,J})\log s}
+ \sigma s\sqrt{ \max\big\{ \Delta_{\max}(\mathcal{G}_{J,J^c}), \Delta_{\max}(\mathcal{G}_{J^c,J^c}) \big\} \log d}
+ s\| \pmb{M}^*_{J^c,J} \|_{2}
+ \frac{1}{\sqrt{s}}\| \pmb{M}^*_{J^c,J^c}\|_2
\\
\leq
\frac{c\phi(\mathcal{G}_{J,J}) \bar{\lambda}(\pmb{M}^*_{J,J}) \cdot \min_{i \in J}|u_{1,i}|}{s},
\end{multline*}
with some constant $c>0$.
Since $\bar{\lambda}(\pmb{M}^*_{J,J}) \geq \bar{\lambda}(\pmb{M}^*)$, we can replace $\bar{\lambda}(\pmb{M}^*_{J,J})$ by $\bar{\lambda}(\pmb{M}^*)$.

\newpage

\section{Proof of Theorem \ref{thm:by_product_tail_bound}}
\label{sec:proof_of_tail_bound}

We make use of the following theorem to prove \cref{thm:by_product_tail_bound}.

\begin{theorem}[Master Tail Bound for Independent Sums (Theorem 3.6 in \citet{tropp2012user})]
\label{thm:tail_bound}
Consider a finite sequence $\{\pmb{Z}_l\}_{l=1}^m$ of independent, random, symmetric matrices.
For all $t \in \mathbb{R}$,
$$
\mathbb{P}\bigg[
\lambda_1\Big( \sum_{l=1}^m \pmb{Z}_l \Big) \geq t \bigg]
\leq 
\inf_{\theta > 0} \bigg\{
e^{-\theta t} \cdot \textup{trexp} \Big(
\sum_{l=1}^m \log \mathbb{E} e^{\theta \pmb{Z}_l} \Big)
\bigg\}.
$$
If $\pmb{Z}_l$ and $-\pmb{Z}_l$ have the same distribution for all $l$, then for any $t \geq 0$,
$$
\mathbb{P}\bigg[
\Big\| \sum_{l=1}^m \pmb{Z}_l \Big\|_2 \geq t \bigg]
\leq 
2\cdot \inf_{\theta > 0} \bigg\{
e^{-\theta t} \cdot \textup{trexp} \Big(
\sum_{l=1}^m \log \mathbb{E} e^{\theta \pmb{Z}_l} \Big)
\bigg\}.
$$
\end{theorem}

The following theorem is a comprehensive version of \cref{thm:by_product_tail_bound}, which includes the result of the symmetric random matrix case.

\begin{theorem}[Tail Bound for Partial Random Matrix with Independent Sub-Gaussian Entries]
\label{thm:by_product_tail_bound_comprehensive}
Consider a $m \times n$ random matrix $\pmb{Z}$ whose subset of entries independently follow sub-Gaussian distributions which are symmetric about zero and have parameter $\sigma > 0$, while the other entries are zero.
That is, there exists an index set $S \subseteq \{(i,j)~|~ i \in [m], j \in [n]\}$ such that for $i \in [m]$ and $j \in [n]$,
$$
Z_{i,j} =  
\begin{cases}
N_{i,j} & \text{if } (i,j)\in S \\
0 & \text{if } (i,j)\notin S
\end{cases}
$$
where each $N_{i,j}$ is symmetric about zero and satisfies $\mathbb{E}e^{\theta N_{i,j}}\leq e^{\frac{\sigma^2 \theta^2}{2}}$ for any $\theta >0$.
Then for any $t \geq 0$,
$$
\mathbb{P} [\| \pmb{Z} \|_2 \geq t ]
\leq 
2(m+n)\cdot \exp\Big(-\frac{t^2}{2\sigma^2 \Delta_{\max}(\mathcal{G}_S)}\Big),
$$
where $\mathcal{G}_S$ is a bipartite graph whose vertex and edge sets are $[m]\times[n]$ and $S$, respectively.
This inequality implies that
$$
\| \pmb{Z} \|_2 
\leq
2\sigma \sqrt{\Delta_{\max}(\mathcal{G}_S) \log (m+n)}
$$
with probability at least $1-2 (m+n)^{-1}$.

If $\pmb{Z}$ is a symmetric matrix with dimension $n$, then for any $t \geq 0$,
$$
\mathbb{P} [\| \pmb{Z} \|_2 \geq t ]
\leq 
2n\cdot \exp\Big(-\frac{t^2}{2\sigma^2 \Delta_{\max}(\mathcal{G}_S)}\Big),
$$
where $\mathcal{G}_S$ is an undirected graph whose vertex and edge sets are $[n]$ and $S$, respectively.
This implies that
$$
\| \pmb{Z} \|_2 
\leq
2\sigma \sqrt{\Delta_{\max}(\mathcal{G}_S) \log n}
$$
with probability at least $1-2 n^{-1}$.
\end{theorem}

\begin{proof}
We first consider the case that $\pmb{Z}$ is a symmetric matrix with dimension $n$.
We can write $\pmb{Z}$ as follows:
\begin{align*}
\pmb{Z}
&= \sum_{i,j \in [n]:(i,j)\in S} N_{i,j} \pmb{e}_i \pmb{e}_j^\top
\\&
= \sum_{\substack{i,j \in [n], i < j: \\(i,j)\in S}} 
\underbrace{
N_{i,j} (\pmb{e}_i \pmb{e}_j^\top + \pmb{e}_j \pmb{e}_i^\top)
}_{=: \pmb{W}_{i,j}}
+
\sum_{\substack{i\in [n]: \\(i,j)\in S}} 
\underbrace{
N_{i,i} \pmb{e}_i \pmb{e}_i^\top
}_{=: \pmb{W}_{i,i}},
\end{align*}
which can be viewed as a sum of independent, symmetric matrices $\{\pmb{W}_{i,j}\}_{i\leq j, (i,j)\in S}$.
We first note that for any $\theta > 0$ and $i,j\in [n]$ such that $i< j$,
\begin{align*}
e^{\theta \pmb{W}_{i,j}} 
&= \pmb{I} + \sum_{k=1}^{\infty} \frac{(\theta \pmb{W}_{i,j})^k}{k!}
\\&= 
\pmb{I} + \sum_{k=1}^{\infty} \frac{(\theta N_{i,j})^{2k}}{(2k)!} (\pmb{e}_i \pmb{e}_j^\top + \pmb{e}_j \pmb{e}_i^\top)^{2k} 
 + \sum_{k=1}^{\infty} \frac{(\theta N_{i,j})^{2k-1}}{(2k-1)!} (\pmb{e}_i \pmb{e}_j^\top + \pmb{e}_j \pmb{e}_i^\top)^{2k-1}
\\&=
\pmb{I} + \sum_{k=1}^{\infty} \frac{(\theta N_{i,j})^{2k}}{(2k)!} (\pmb{e}_i \pmb{e}_i^\top + \pmb{e}_j \pmb{e}_j^\top) 
 + \sum_{k=1}^{\infty} \frac{(\theta N_{i,j})^{2k-1}}{(2k-1)!} (\pmb{e}_i \pmb{e}_j^\top + \pmb{e}_j \pmb{e}_i^\top)
\\&=
\pmb{I} + \bigg(\frac{e^{\theta N_{i,j}}+e^{-\theta N_{i,j}}}{2} -1\bigg)\cdot (\pmb{e}_i \pmb{e}_i^\top + \pmb{e}_j \pmb{e}_j^\top) 
 + \bigg(\frac{e^{\theta N_{i,j}}-e^{-\theta N_{i,j}}}{2} \bigg)\cdot (\pmb{e}_i \pmb{e}_j^\top + \pmb{e}_j \pmb{e}_i^\top),
\end{align*}
and for $i \in [n]$,
\begin{align*}
e^{\theta \pmb{W}_{i,i}} 
&= \pmb{I} + \sum_{k=1}^{\infty} \frac{(\theta \pmb{W}_{i,i})^k}{k!}
= 
\pmb{I} + \sum_{k=1}^{\infty} \frac{(\theta N_{i,i})^{k}}{k!} (\pmb{e}_i \pmb{e}_i^\top)^{k}
\\&=
\pmb{I} + \sum_{k=1}^{\infty} \frac{(\theta N_{i,i})^{k}}{k!} \pmb{e}_i \pmb{e}_i^\top
=
\pmb{I} + \Big( e^{\theta N_{i,i}} -1 \Big)\cdot \pmb{e}_i \pmb{e}_i^\top.
\end{align*}
These quantities have the expectations as follows: 
\begin{align*}
\mathbb{E}e^{\theta \pmb{W}_{i,j}} 
&= \pmb{I} + (\mathbb{E}e^{\theta N_{i,j}} -1)\cdot (\pmb{e}_i \pmb{e}_i^\top + \pmb{e}_j \pmb{e}_j^\top) 
\\
\mathbb{E}e^{\theta \pmb{W}_{i,i}} 
&= \pmb{I} + (\mathbb{E}e^{\theta N_{i,i}} -1)\cdot \pmb{e}_i \pmb{e}_i^\top
\end{align*}
where the fact that $\mathbb{E}e^{\theta N_{i,j}} = \mathbb{E}e^{-\theta N_{i,j}}$ is used, which is because each $N_{i,j}$ is symmetric about zero.
Note that each $\mathbb{E}e^{\theta \pmb{W}_{i,j}}$ ($\mathbb{E}e^{\theta \pmb{W}_{i,i}}$, resp.) is a diagonal matrix whose $i$-th and $j$-th ($i$-th, resp.) diagonal entries are $\mathbb{E}e^{\theta N_{i,j}}$ ($\mathbb{E}e^{\theta N_{i,i}}$, resp.) while the other diagonal entries are 1.
Now we can write the summation of the logarithms of the expectations as follows:
\begin{align*}
\sum_{\substack{i,j \in [n], i < j: \\(i,j)\in S}} 
\log \mathbb{E}e^{\theta \pmb{W}_{i,j}}
+
\sum_{\substack{i\in [n]: \\(i,j)\in S}} 
\log \mathbb{E}e^{\theta \pmb{W}_{i,i}} 
&=
\log \Bigg( \prod_{\substack{i,j \in [n], i < j: \\(i,j)\in S}} \mathbb{E}e^{\theta \pmb{W}_{i,j}} \cdot
\prod_{\substack{i\in [n]: \\(i,j)\in S}} \mathbb{E}e^{\theta \pmb{W}_{i,i}} \Bigg)
\\&=
\log \Bigg( \sum_{i\in [n]} \bigg( \prod_{j \in [n], (i,j)\in S} \mathbb{E}e^{\theta N_{i,j}} \bigg) \cdot \pmb{e}_i \pmb{e}_i^\top \Bigg)
\end{align*}
where the first equality holds because $\mathbb{E}e^{\theta \pmb{W}_{i,j}}$'s and $\mathbb{E}e^{\theta \pmb{W}_{i,i}}$'s are positive definite and commute.
Hence,
\begin{multline*}
\text{trexp} \bigg(
\sum_{\substack{i,j \in [n], i < j: \\(i,j)\in S}} 
\log \mathbb{E}e^{\theta \pmb{W}_{i,j}}
+
\sum_{\substack{i\in [n]: \\(i,j)\in S}} 
\log \mathbb{E}e^{\theta \pmb{W}_{i,i}} 
\bigg)
=
\text{trexp} \log \Bigg( \sum_{i\in [n]} \bigg( \prod_{j \in [n], (i,j)\in S} \mathbb{E}e^{\theta N_{i,j}} \bigg) \cdot \pmb{e}_i \pmb{e}_i^\top \Bigg)
\\=
\text{tr} \Bigg( \sum_{i\in [n]} \bigg( \prod_{j \in [n], (i,j)\in S} \mathbb{E}e^{\theta N_{i,j}} \bigg) \cdot \pmb{e}_i \pmb{e}_i^\top \Bigg)
= 
\sum_{i\in [n]} \bigg( \prod_{j \in [n], (i,j)\in S} \mathbb{E}e^{\theta N_{i,j}} \bigg).
\end{multline*}
Therefore, we have that
\begin{align*}
\inf_{\theta > 0} \bigg\{
e^{-\theta t} \cdot 
\text{trexp} \bigg(
\sum_{\substack{i,j \in [n], i < j: \\(i,j)\in S}} 
\log \mathbb{E}e^{\theta \pmb{W}_{i,j}}
+
\sum_{\substack{i\in [n]: \\(i,j)\in S}} 
\log \mathbb{E}e^{\theta \pmb{W}_{i,i}} 
\bigg) \bigg\}
=
\inf_{\theta > 0} \bigg\{
e^{-\theta t} \cdot \sum_{i\in [n]} \bigg( \prod_{j \in [n], (i,j)\in S} \mathbb{E}e^{\theta N_{i,j}} \bigg)\bigg\}.
\end{align*}
Since $\mathbb{E}e^{\theta N_{i,j}} \leq e^{\frac{\sigma^2 \theta^2}{2}}$ for any $\theta >0$ and $i,j \in [n]$, we can derive that
\begin{align*}
&\inf_{\theta > 0} \bigg\{
e^{-\theta t} \cdot \sum_{i\in [n]} \bigg( \prod_{j \in [n], (i,j)\in S} \mathbb{E}e^{\theta N_{i,j}} \bigg)\bigg\}
\leq
\inf_{\theta > 0} \bigg\{
e^{-\theta t} \cdot \sum_{i\in [n]}   \exp\Big(\frac{\sigma^2 \theta^2
\#\{j \in [n] ~;~ (i,j)\in S\}
}{2}\Big) \bigg\}
\\&\leq
\inf_{\theta > 0} \bigg\{
e^{-\theta t} \cdot n \cdot \exp\Big(\frac{\sigma^2 \theta^2  \max_{i\in [n]}\#\{j \in [n] ~;~ (i,j)\in S\} }{2}\Big) \bigg\}
=
\inf_{\theta > 0} \bigg\{
n \cdot \exp\Big(\frac{\sigma^2 \theta^2  \Delta_{\max}(\mathcal{G}_S)}{2}-\theta t\Big) \bigg\}
\\&=
n\cdot \exp\Big(-\frac{t^2}{2\sigma^2 \Delta_{\max}(\mathcal{G}_S)}\Big).
\end{align*}
Therefore, by Theorem \ref{thm:tail_bound},
\begin{align*}
\mathbb{P}\big[
\| \pmb{Z} \|_2 \geq t \big]
\leq 
2n\cdot \exp\Big(-\frac{t^2}{2\sigma^2 \Delta_{\max}(\mathcal{G}_S)}\Big).
\end{align*}

Next, when $\pmb{Z}$ is $m\times n$ matrix, we use the fact that
$\| \pmb{Z} \|_2
=
\bigg\|\begin{pmatrix}
\pmb{O} & \pmb{Z} \\
\pmb{Z}^\top & \pmb{O}
\end{pmatrix}
\bigg\|_2.
$
We can write that
\begin{align*}
\begin{pmatrix}
\pmb{O} & \pmb{Z} \\
\pmb{Z}^\top & \pmb{O}
\end{pmatrix}
&=
\sum_{\substack{i\in [m], j \in [n]: \\(i,j)\in S}} 
\underbrace{
N_{i,j} (\pmb{e}_i \pmb{e}_j^\top + \pmb{e}_j \pmb{e}_i^\top)
}_{=: \pmb{W}_{i,j}}
\end{align*}
which can be viewed as a sum of independent, symmetric matrices $\{\pmb{W}_{i,j}\}_{i\in [m], j \in [n], (i,j)\in S}$.
As we have shown before,  
$\mathbb{E}e^{\theta \pmb{W}_{i,j}} 
= \pmb{I} + (\mathbb{E}e^{\theta N_{i,j}} -1)\cdot (\pmb{e}_i \pmb{e}_i^\top + \pmb{e}_j \pmb{e}_j^\top)$, and we can derive that
\begin{align*}
\sum_{\substack{i\in [m], j \in [n]: \\(i,j)\in S}} 
\log \mathbb{E}e^{\theta \pmb{W}_{i,j}} 
&=
\log \bigg( \prod_{\substack{i\in [m], j \in [n]: \\(i,j)\in S}} \mathbb{E}e^{\theta \pmb{W}_{i,j}} \bigg) 
\\&=
\log \Bigg( \sum_{i\in [m]} \bigg[ \prod_{j \in [n], (i,j)\in S} \mathbb{E}e^{\theta N_{i,j}} \bigg] \cdot \pmb{e}_i \pmb{e}_i^\top 
+
\sum_{i\in [n]} \bigg[ \prod_{j \in [m], (i,j)\in S} \mathbb{E}e^{\theta N_{i,j}} \bigg] \cdot \pmb{e}_i \pmb{e}_i^\top \Bigg).
\end{align*}
Hence,
\begin{align*}
\text{trexp} \bigg(
\sum_{\substack{i\in [m], j \in [n]: \\(i,j)\in S}} 
\log \mathbb{E}e^{\theta \pmb{W}_{i,j}} 
\bigg)
&=
\text{tr} \Bigg( \sum_{i\in [m]} \bigg[ \prod_{j \in [n], (i,j)\in S} \mathbb{E}e^{\theta N_{i,j}} \bigg] \cdot \pmb{e}_i \pmb{e}_i^\top 
+
\sum_{i\in [n]} \bigg[ \prod_{j \in [m], (i,j)\in S} \mathbb{E}e^{\theta N_{i,j}} \bigg] \cdot \pmb{e}_i \pmb{e}_i^\top \Bigg)
\\&= 
\sum_{i\in [m]} \bigg[ \prod_{j \in [n], (i,j)\in S} \mathbb{E}e^{\theta N_{i,j}} \bigg]
+
\sum_{i\in [n]} \bigg[ \prod_{j \in [m], (i,j)\in S} \mathbb{E}e^{\theta N_{i,j}} \bigg],
\end{align*}
and we have that
\begin{align*}
\inf_{\theta > 0} \bigg\{
e^{-\theta t} \cdot 
\text{trexp} \bigg(
\sum_{\substack{i\in [m], j \in [n]: \\(i,j)\in S}} 
\log \mathbb{E}e^{\theta \pmb{W}_{i,j}} 
\bigg) \bigg\}
=
\inf_{\theta > 0} \Bigg\{
e^{-\theta t} \cdot \Bigg( 
\sum_{i\in [m]} \bigg[ \prod_{j \in [n], (i,j)\in S} \mathbb{E}e^{\theta N_{i,j}} \bigg]
+
\sum_{i\in [n]} \bigg[ \prod_{j \in [m], (i,j)\in S} \mathbb{E}e^{\theta N_{i,j}} \bigg]
 \Bigg)\Bigg\}.
\end{align*}
Since $\mathbb{E}e^{\theta N_{i,j}} \leq e^{\frac{\sigma^2 \theta^2}{2}}$ for any $\theta >0$ and $i,j$, we can derive that
\begin{align*}
&\inf_{\theta > 0} \Bigg\{
e^{-\theta t} \cdot \Bigg( 
\sum_{i\in [m]} \bigg[ \prod_{j \in [n], (i,j)\in S} \mathbb{E}e^{\theta N_{i,j}} \bigg]
+
\sum_{i\in [n]} \bigg[ \prod_{j \in [m], (i,j)\in S} \mathbb{E}e^{\theta N_{i,j}} \bigg]
 \Bigg)\Bigg\}
\\&\leq
\inf_{\theta > 0} \bigg\{
e^{-\theta t} \cdot 
\bigg[m \cdot \exp\Big(\frac{\sigma^2 \theta^2  \max_{i\in [m]}
\#\{j\in [n] ~;~ (i,j)\in S\}
}{2}\Big) 
+ n \cdot \exp\Big(\frac{\sigma^2 \theta^2  \max_{i\in [n]}
\#\{j\in [m] ~;~ (i,j)\in S\}}{2}\Big) 
\bigg]
\bigg\}
\\&\leq
\inf_{\theta > 0} \bigg\{
e^{-\theta t} \cdot 
(m+n) \cdot \exp\Big(\frac{\sigma^2 \theta^2  
\Delta_{\max}(\mathcal{G}_S)
}{2}\Big) 
\bigg\}
=
(m+n)\cdot \exp\Big(-\frac{t^2}{2\sigma^2 \Delta_{\max}(\mathcal{G}_S)}\Big).
\end{align*}
Therefore, by Theorem \ref{thm:tail_bound},
$$
\mathbb{P}\big[
\| \pmb{Z} \|_2 \geq t \big]
\leq 
2(m+n)\cdot \exp\Big(-\frac{t^2}{2\sigma^2 \Delta_{\max}(\mathcal{G}_S)}\Big).
$$
\end{proof}

\begin{lemma}
\label{lemma_tail_bound}
When each $N_{i,j}$ is symmetric about zero and satisfies $\mathbb{E}e^{\theta N_{i,j}}\leq e^{\frac{\sigma^2 \theta^2}{2}}$ for any $\theta >0$,
\begin{align*}
&\| \mathbb{E}[\pmb{M}_{J,J}] - \pmb{M}_{J,J}\|_2
= \|(\pmb{A}_\mathcal{G})_{J,J}\circ \pmb{N}_{J,J} \|_2
\leq
2\sigma \sqrt{ \Delta_{\max}(\mathcal{G}_{J,J}) \log s}
~~~~~~~~~~~~~~~~~~~~~~~~~~\text{with probability at least}~~  1-2 s^{-1},
\\
&\| \mathbb{E}[\pmb{M}_{J^c,J}] - \pmb{M}_{J^c,J}\|_2 
= \|(\pmb{A}_\mathcal{G})_{J^c,J}\circ \pmb{N}_{J^c,J} \|_2
\leq
2\sigma \sqrt{ \Delta_{\max}(\mathcal{G}_{J,J^c}) \log d}
~~~~~~~~~~~~~~~~~\text{with probability at least}~~  1-2 d^{-1},
\\
&\| \mathbb{E}[\pmb{M}_{J^c,J^c}] - \pmb{M}_{J^c,J^c} \|_2
= \|(\pmb{A}_\mathcal{G})_{J^c,J^c}\circ \pmb{N}_{J^c,J^c} \|_2
\leq
2\sigma \sqrt{ \Delta_{\max}(\mathcal{G}_{J^c,J^c}) \log d }
~~~~~~~~~\text{with probability at least}~~  1-2 d^{-1}.
\end{align*}

\end{lemma}

\begin{proof}
Straightforwardly, the inequalities are obtained by invoking \cref{thm:by_product_tail_bound_comprehensive}.
\end{proof}

\newpage

\section{Proof of Theorem \ref{thm:by_product_diff_bound}}
\label{sec:proof_of_diff_bound}

For simplicity, let $\phi = \phi(\mathcal{G})$ and $\psi = \psi(\mathcal{G})$ in this proof.
First, note that
\begin{align}
\| \pmb{Y} - \frac{n}{\phi}\cdot\pmb{A}_\mathcal{G}\circ \pmb{Y} \|_2
&=
\max_{\|\pmb{y}\|_2 = 1} \bigg|\pmb{y}^\top \bigg\{ \sum_{k\in[r]} \lambda_k(\pmb{Y}) \pmb{v}_k \pmb{v}_k^\top - \frac{n}{\phi} \sum_{k\in[r]} \lambda_k(\pmb{Y}) (\pmb{v}_k \pmb{v}_k^\top \circ \pmb{A}_\mathcal{G}) \bigg\} \pmb{y} \bigg|
\nonumber
\\&\leq
\max_{\|\pmb{y}\|_2 = 1} \sum_{k\in[r]}  |\lambda_k(\pmb{Y})|\cdot
\Big| \pmb{y}^\top \Big\{ \pmb{v}_k \pmb{v}_k^\top - \frac{n}{\phi} (\pmb{v}_k \pmb{v}_k^\top \circ \pmb{A}_\mathcal{G}) \Big\} \pmb{y} \Big|
\nonumber
\\&=
\max_{\|\pmb{y}\|_2 = 1} \sum_{k\in[r]}  |\lambda_k(\pmb{Y})|\cdot
\Big| 
(\pmb{y}^\top \pmb{v}_k)^2 - \frac{n}{\phi} (\pmb{y} \circ \pmb{v}_k )^\top \pmb{A}_\mathcal{G} (\pmb{y} \circ \pmb{v}_k ) \Big|.
\label{ub_a1}
\end{align}
Now, we will find the upper and lower bounds of $(\pmb{y}^\top \pmb{v}_k)^2 - \frac{n}{\phi} (\pmb{y} \circ \pmb{v}_k )^\top \pmb{A}_\mathcal{G} (\pmb{y} \circ \pmb{v}_k )$. 
Note that we can write $\pmb{y} \circ \pmb{v}_k  = \frac{\pmb{y}^\top \pmb{v}_k}{n} \cdot \pmb{1} + \sqrt{(\pmb{y} \circ \pmb{v}_k )^\top(\pmb{y} \circ \pmb{v}_k ) - \frac{(\pmb{y}^\top \pmb{v}_k)^2}{n}} \cdot \pmb{1}_\perp$
with $\pmb{1} = (1, 1, \dots, 1)^\top \in \mathbb{R}^{n}$ and $\pmb{1}_\perp\in \mathbb{R}^{n}$ where $\pmb{1}_\perp$ is some unit vector orthogonal to $\pmb{1}$.
First, we derive the lower bound of $(\pmb{y}^\top \pmb{v}_k)^2 - \frac{n}{\phi} (\pmb{y} \circ \pmb{v}_k )^\top \pmb{A}_\mathcal{G} (\pmb{y} \circ \pmb{v}_k )$ as follows:

\begin{align}
&(\pmb{y}^\top \pmb{v}_k)^2 - \frac{n}{\phi} (\pmb{y} \circ \pmb{v}_k )^\top \pmb{A}_\mathcal{G} (\pmb{y} \circ \pmb{v}_k )
=
(\pmb{y}^\top \pmb{v}_k)^2 
- \frac{n}{\phi} (\pmb{y} \circ \pmb{v}_k )^\top (\pmb{A}_\mathcal{G}-\pmb{D}_\mathcal{G}+\pmb{D}_\mathcal{G}) (\pmb{y} \circ \pmb{v}_k )
\nonumber
\\&=
(\pmb{y}^\top \pmb{v}_k)^2 
+ \frac{n}{\phi} (\pmb{y} \circ \pmb{v}_k )^\top (\pmb{D}_\mathcal{G}-\pmb{A}_\mathcal{G}) (\pmb{y} \circ \pmb{v}_k )
- \frac{n}{\phi} (\pmb{y} \circ \pmb{v}_k )^\top \pmb{D}_\mathcal{G} (\pmb{y} \circ \pmb{v}_k )
\nonumber
\\&=
(\pmb{y}^\top \pmb{v}_k)^2 
+ \frac{n}{\phi} \cdot \Big\{ (\pmb{y} \circ \pmb{v}_k )^\top(\pmb{y} \circ \pmb{v}_k ) - \frac{(\pmb{y}^\top \pmb{v}_k)^2}{n} \Big\} \pmb{1}_\perp^\top (\pmb{D}_\mathcal{G}-\pmb{A}_\mathcal{G}) \pmb{1}_\perp
- \frac{n}{\phi} (\pmb{y} \circ \pmb{v}_k )^\top \pmb{D}_\mathcal{G} (\pmb{y} \circ \pmb{v}_k )
\nonumber
\\&\geq
(\pmb{y}^\top \pmb{v}_k)^2 
+ \frac{n}{\phi} \cdot \Big\{ (\pmb{y} \circ \pmb{v}_k )^\top(\pmb{y} \circ \pmb{v}_k ) - \frac{(\pmb{y}^\top \pmb{v}_k)^2}{n} \Big\} \phi
- \frac{n}{\phi} (\pmb{y} \circ \pmb{v}_k )^\top (\pmb{y} \circ \pmb{v}_k )\Delta_{\max}
\nonumber
\\&=
\frac{n(\phi - \Delta_{\max})}{\phi}\cdot (\pmb{y} \circ \pmb{v}_k )^\top(\pmb{y} \circ \pmb{v}_k )
\label{1st_ub_a1}
\end{align}
where $\Delta_{\max} = \Delta_{\max}(\mathcal{G})$ and $\pmb{D}_\mathcal{G}$ is a diagonal matrix whose diagonal entries are the node degrees of $\pmb{A}_\mathcal{G}$.
Similarly, we can derive the upper bound as follows:
\begin{align}
&(\pmb{y}^\top \pmb{v}_k)^2 - \frac{n}{\phi} (\pmb{y} \circ \pmb{v}_k )^\top \pmb{A}_\mathcal{G} (\pmb{y} \circ \pmb{v}_k )
\nonumber
\\&=
(\pmb{y}^\top \pmb{v}_k)^2 
- \frac{n}{\phi} (\pmb{y} \circ \pmb{v}_k )^\top (\pmb{A}_\mathcal{G}-\pmb{1}\pmb{1}^\top+\pmb{1}\pmb{1}^\top + n \pmb{I} - \pmb{D}_\mathcal{G}-n\pmb{I}+\pmb{D}_\mathcal{G}) (\pmb{y} \circ \pmb{v}_k )
\nonumber
\\&=
(\pmb{y}^\top \pmb{v}_k)^2 
- \frac{n}{\phi} (\pmb{y} \circ \pmb{v}_k )^\top \pmb{1}\pmb{1}^\top (\pmb{y} \circ \pmb{v}_k ) - \frac{n}{\phi} (\pmb{y} \circ \pmb{v}_k )^\top ( \pmb{A}_\mathcal{G}-\pmb{1}\pmb{1}^\top + n \pmb{I} - \pmb{D}_\mathcal{G} ) (\pmb{y} \circ \pmb{v}_k ) 
\nonumber
\\&~~~
+ \frac{n}{\phi} (\pmb{y} \circ \pmb{v}_k )^\top ( n \pmb{I} - \pmb{D}_\mathcal{G} ) (\pmb{y} \circ \pmb{v}_k )
\nonumber
\\&=
\frac{\phi - n}{\phi} (\pmb{y}^\top \pmb{v}_k)^2 
-\frac{n}{\phi} \cdot \Big\{ (\pmb{y} \circ \pmb{v}_k )^\top(\pmb{y} \circ \pmb{v}_k ) - \frac{(\pmb{y}^\top \pmb{v}_k)^2}{n} \Big\} \pmb{1}_\perp^\top (\pmb{A}_\mathcal{G}-\pmb{1}\pmb{1}^\top + n \pmb{I} - \pmb{D}_\mathcal{G} ) \pmb{1}_\perp
\nonumber
\\&~~~
+ \frac{n}{\phi} (\pmb{y} \circ \pmb{v}_k )^\top ( n \pmb{I} - \pmb{D}_\mathcal{G} ) (\pmb{y} \circ \pmb{v}_k )
\nonumber
\\&\leq
\frac{\phi - n}{\phi} (\pmb{y}^\top \pmb{v}_k)^2 
-\frac{n}{\phi} \cdot \Big\{ (\pmb{y} \circ \pmb{v}_k )^\top(\pmb{y} \circ \pmb{v}_k ) - \frac{(\pmb{y}^\top \pmb{v}_k)^2}{n} \Big\} \phi(\overline{\mathcal{G}})
+ \frac{n}{\phi} (\pmb{y} \circ \pmb{v}_k )^\top (\pmb{y} \circ \pmb{v}_k )(n-\Delta_{\min})
\nonumber
\\&=
\frac{\phi - (n - \phi(\overline{\mathcal{G}}))}{\phi} (\pmb{y}^\top \pmb{v}_k)^2 
+ \frac{n(\Delta_{\max}(\overline{\mathcal{G}})-\phi(\overline{\mathcal{G}}))}{\phi}(\pmb{y} \circ \pmb{v}_k )^\top (\pmb{y} \circ \pmb{v}_k )
\label{2nd_up_a1}
\end{align}
where $\pmb{I}\in \mathbb{R}^{n\times n}$ is an identity matrix and $\pmb{1} = (1, 1, \dots, 1)^\top \in \mathbb{R}^{n}$.

We can use \eqref{1st_ub_a1} and \eqref{2nd_up_a1} to derive the upper bound of \eqref{ub_a1}. First, by using \eqref{1st_ub_a1}, we have
\begin{align*}
&\sum_{k\in[r]}  |\lambda_k(\pmb{Y})|\cdot
\Big| 
(\pmb{y}^\top \pmb{v}_k)^2 - \frac{n}{\phi} (\pmb{y} \circ \pmb{v}_k )^\top \pmb{A}_\mathcal{G} (\pmb{y} \circ \pmb{v}_k ) \Big|
\leq
\sum_{k\in[r]}  |\lambda_k(\pmb{Y})|\cdot
\Big| 
\frac{n(\phi - \Delta_{\max})}{\phi}\cdot (\pmb{y} \circ \pmb{v}_k )^\top(\pmb{y} \circ \pmb{v}_k ) \Big|
\\&=
\frac{n(\Delta_{\max}-\phi)}{\phi}\cdot \sum_{k\in[r]}  |\lambda_k(\pmb{Y})|\cdot (\pmb{y} \circ \pmb{v}_k )^\top(\pmb{y} \circ \pmb{v}_k )
\\&\leq 
\frac{n(\Delta_{\max}-\phi)}{\phi}\cdot \|\pmb{Y}\|_2 \cdot \sum_{k\in[r]}  \sum_{i\in [n]} y_i^2 v_{k,i}^2
=
\frac{n(\Delta_{\max}-\phi)}{\phi}\cdot \|\pmb{Y}\|_2 \cdot \sum_{i\in [n]} y_i^2 \sum_{k\in[r]} v_{k,i}^2
\\&\leq
\frac{n(\Delta_{\max}-\phi)}{\phi}\cdot \|\pmb{Y}\|_2 \cdot\tau.
\end{align*}
Also, with \eqref{2nd_up_a1}, we can derive that
\begin{align*}
&\sum_{k\in[r]}  |\lambda_k(\pmb{Y})|\cdot
\Big| 
(\pmb{y}^\top \pmb{v}_k)^2 - \frac{n}{\phi} (\pmb{y} \circ \pmb{v}_k )^\top \pmb{A}_\mathcal{G} (\pmb{y} \circ \pmb{v}_k ) \Big|
\\&\leq
\sum_{k\in[r]}  |\lambda_k(\pmb{Y})|\cdot
\Big| 
\frac{\phi - (n - \phi(\overline{\mathcal{G}}))}{\phi} (\pmb{y}^\top \pmb{v}_k)^2 
+ \frac{n(\Delta_{\max}(\overline{\mathcal{G}})-\phi(\overline{\mathcal{G}}))}{\phi}(\pmb{y} \circ \pmb{v}_k )^\top (\pmb{y} \circ \pmb{v}_k )
\Big|
\\&=
\sum_{k\in[r]}  |\lambda_k(\pmb{Y})|\cdot
\bigg| 
(\pmb{y} \circ \pmb{v}_k )^\top
\bigg\{
\frac{\phi - (n - \phi(\overline{\mathcal{G}}))}{\phi} \pmb{1}\pmb{1}^\top
+ \frac{n(\Delta_{\max}(\overline{\mathcal{G}})-\phi(\overline{\mathcal{G}}))}{\phi} \pmb{I}
\bigg\}
(\pmb{y} \circ \pmb{v}_k )
\bigg|
\\&\leq
\sum_{k\in[r]}  |\lambda_k(\pmb{Y})|\cdot 
(\pmb{y} \circ \pmb{v}_k )^\top(\pmb{y} \circ \pmb{v}_k )
\cdot
\bigg\|
\frac{\phi - (n - \phi(\overline{\mathcal{G}}))}{\phi} \pmb{1}\pmb{1}^\top
+ \frac{n(\Delta_{\max}(\overline{\mathcal{G}})-\phi(\overline{\mathcal{G}}))}{\phi} \pmb{I}
\bigg\|_2
\\&\leq
\|\pmb{Y}\|_2 \cdot\tau \cdot 
\max\bigg\{
\bigg| \frac{\phi - (n - \phi(\overline{\mathcal{G}}))}{\phi} + \frac{n(\Delta_{\max}(\overline{\mathcal{G}})-\phi(\overline{\mathcal{G}}))}{\phi} \bigg|,~
\frac{n(\Delta_{\max}(\overline{\mathcal{G}})-\phi(\overline{\mathcal{G}}))}{\phi}
\bigg\}
\\&=
\|\pmb{Y}\|_2 \cdot\tau \cdot \frac{n(\Delta_{\max}(\overline{\mathcal{G}})-\phi(\overline{\mathcal{G}}))}{\phi}
\end{align*}
where the last equality is due to the fact that $\phi \leq n - \phi(\overline{\mathcal{G}})$ always.

Therefore, we have the upper bound of $\| \pmb{Y} - \frac{n}{\phi}\cdot\pmb{A}_\mathcal{G}\circ \pmb{Y} \|_2$, which is
$$
\| \pmb{Y} - \frac{n}{\phi}\cdot\pmb{A}_\mathcal{G}\circ \pmb{Y} \|_2
\leq
\frac{n}{\phi}\cdot \tau \|\pmb{Y}\|_2 \cdot \max\{\Delta_{\max} - \phi,~ \Delta_{\max}(\overline{\mathcal{G}})-\phi(\overline{\mathcal{G}}) \}
= \frac{n\tau\psi}{\phi} \cdot \|\pmb{Y}\|_2.
$$

\newpage

\section{Auxiliary Lemmas}

\begin{lemma}
\label{lemma_sign}
For any unit vectors $\pmb{x} \in \mathbb{R}^d$ and $\pmb{y} \in \mathbb{R}^d$ such that $y_i \neq 0$ for $\forall i\in[d]$, if $\|\pmb{x}-\pmb{y}\|_2 \leq \min_{i\in[d]} |y_i|$, then $sign(x_i) = sign(y_i)$ for $\forall i\in[d]$.
\end{lemma}

\begin{proof}
If $\pmb{x}=\pmb{y}$, then it is trivial that $sign(x_i) = sign(y_i)$ for $\forall i\in[d]$.
If $\pmb{x}\neq \pmb{y}$, then
for any $i\in[d]$,
$$
|x_i - y_i| < \|\pmb{x} - \pmb{y}\|_2 \leq \min_{i\in[d]} |y_i| \leq |y_i|,
$$
where the first inequality is strict since both $\pmb{x}$ and $\pmb{y}$ are unit vectors.
The above inequality implies that
$$
y_i - |y_i| < x_i < y_i + |y_i|,
$$
that is,
$0 < x_i < 2y_i$ if $y_i > 0$, and $2y_i < x_i < 0$ if $y_i < 0$.
Therefore, $sign(x_i) = sign(y_i)$ holds for any $i\in[d]$.
\end{proof}

\begin{lemma}
\label{lemma_eig_ineq}
If the following inequality holds:
$$
\| \pmb{M}_{J^c,J}-\rho \hat{\pmb{Z}}_{J^c,J} \|_2^2 
\leq 
\big\{ \lambda_1(\pmb{M}_{J,J}-\rho \hat{\pmb{z}} \hat{\pmb{z}}^\top) - \lambda_2(\pmb{M}_{J,J}-\rho \hat{\pmb{z}} \hat{\pmb{z}}^\top) \big\} 
\cdot \big\{\lambda_1(\pmb{M}_{J,J}-\rho \hat{\pmb{z}} \hat{\pmb{z}}^\top) 
- \| \pmb{M}_{J^c,J^c}-\rho \hat{\pmb{Z}}_{J^c,J^c} \|_2 \big\},
$$
then
$\lambda_1(\pmb{M}_{J,J}-\rho \hat{\pmb{z}} \hat{\pmb{z}}^\top) = \lambda_1(\pmb{M}-\rho \hat{\pmb{Z}})$ 
where
$\hat{\pmb{Z}} = \begin{pmatrix}
\hat{\pmb{z}} \hat{\pmb{z}}^\top & \hat{\pmb{Z}}_{J^c,J}^\top \\
\hat{\pmb{Z}}_{J^c,J} & \hat{\pmb{Z}}_{J^c,J^c}
\end{pmatrix}$. 
\end{lemma}

\begin{proof}
First, we can show that 
$\lambda_1(\pmb{M}_{J,J}-\rho \hat{\pmb{z}} \hat{\pmb{z}}^\top)$ 
is an eigenvalue of $\pmb{M}-\rho \hat{\pmb{Z}}$
where its corresponding eigenvector is 
$(\hat{\pmb{x}}^\top, 0^\top)^\top \in \mathbb{R}^d$.
This is because
$$
(\pmb{M}-\rho \hat{\pmb{Z}})\begin{pmatrix} \hat{\pmb{x}} \\ 0 \end{pmatrix}
=
\begin{pmatrix}
(\pmb{M}_{J,J}-\rho \hat{\pmb{z}} \hat{\pmb{z}}^\top) \hat{\pmb{x}} \\
(\pmb{M}_{J^c,J}-\rho \hat{\pmb{Z}}_{J^c,J}) \hat{\pmb{x}}
\end{pmatrix}
= \lambda_1(\pmb{M}_{J,J}-\rho \hat{\pmb{z}} \hat{\pmb{z}}^\top) \cdot
\begin{pmatrix} \hat{\pmb{x}} \\ 0 \end{pmatrix}
$$
where the last equality holds since $\hat{\pmb{x}}$ is the leading eigenvector of $\pmb{M}_{J,J}-\rho \hat{\pmb{z}} \hat{\pmb{z}}^\top$ and
\begin{align*}
(\pmb{M}_{J^c,J}-\rho \hat{\pmb{Z}}_{J^c,J}) \hat{\pmb{x}}
= 
\pmb{M}_{J^c,J} \hat{\pmb{x}}
- \rho \cdot \frac{1}{\rho \|\hat{\pmb{x}}\|_1} \pmb{M}_{J^c, J}\hat{\pmb{x}} 
\cdot \|\hat{\pmb{x}}\|_1
= 0.
\end{align*}

Now, it is sufficient to show that for all $\pmb{y} = (\pmb{y}_1^\top,  ~\pmb{y}_2^\top )^\top$ such that $\pmb{y}_1 \in \mathbb{R}^s$, $\pmb{y}_2 \in \mathbb{R}^{d-s}$, $\|\pmb{y}_1\|_2^2 + \|\pmb{y}_2\|_2^2 = 1$ and $\hat{\pmb{x}}^\top \pmb{y_1} = 0$,
$$
\pmb{y}^\top (\pmb{M}-\rho \hat{\pmb{Z}}) \pmb{y} 
\leq 
\lambda_1(\pmb{M}_{J,J}-\rho \hat{\pmb{z}} \hat{\pmb{z}}^\top),
$$
which implies that $\lambda_1(\pmb{M}_{J,J}-\rho \hat{\pmb{z}} \hat{\pmb{z}}^\top)$ is the largest eigenvalue of $\pmb{M}-\rho \hat{\pmb{Z}}$.
Note that 
\begin{align*}
&\pmb{y}^\top (\pmb{M}-\rho \hat{\pmb{Z}}) \pmb{y}
= \pmb{y}_1^\top (\pmb{M}_{J,J}-\rho \hat{\pmb{z}} \hat{\pmb{z}}^\top) \pmb{y}_1
+ 
2\pmb{y}_2^\top (\pmb{M}_{J^c,J}-\rho \hat{\pmb{Z}}_{J^c,J}) \pmb{y}_1
+ 
\pmb{y}_2^\top (\pmb{M}_{J^c,J^c}-\rho \hat{\pmb{Z}}_{J^c,J^c}) \pmb{y}_2
\\ &\leq
\lambda_2(\pmb{M}_{J,J}-\rho \hat{\pmb{z}} \hat{\pmb{z}}^\top)\cdot \|\pmb{y}_1\|_2^2
+ 2 \| \pmb{M}_{J^c,J}-\rho \hat{\pmb{Z}}_{J^c,J} \|_2\cdot \|\pmb{y}_1\|_2\cdot \|\pmb{y}_2\|_2
+ \lambda_1( \pmb{M}_{J^c,J^c}-\rho \hat{\pmb{Z}}_{J^c,J^c} )\cdot \|\pmb{y}_2\|_2^2
\\ &=
\lambda_2(\pmb{M}_{J,J}-\rho \hat{\pmb{z}} \hat{\pmb{z}}^\top)\cdot 
(1-\|\pmb{y}_2\|_2^2)
+ 2 \| \pmb{M}_{J^c,J}-\rho \hat{\pmb{Z}}_{J^c,J} \|_2 \cdot 
\sqrt{1-\|\pmb{y}_2\|_2^2} \cdot \|\pmb{y}_2\|_2
+ \lambda_1( \pmb{M}_{J^c,J^c}-\rho \hat{\pmb{Z}}_{J^c,J^c} )\cdot \|\pmb{y}_2\|_2^2
\\ &=
\lambda_2(\pmb{M}_{J,J}-\rho \hat{\pmb{z}} \hat{\pmb{z}}^\top)
+ (\lambda_1( \pmb{M}_{J^c,J^c}-\rho \hat{\pmb{Z}}_{J^c,J^c} ) - \lambda_2(\pmb{M}_{J,J}-\rho \hat{\pmb{z}} \hat{\pmb{z}}^\top)) \cdot \|\pmb{y}_2\|_2^2
\\ &~~~+ 2 \| \pmb{M}_{J^c,J}-\rho \hat{\pmb{Z}}_{J^c,J} \|_2 \cdot 
\sqrt{\|\pmb{y}_2\|_2^2\cdot(1-\|\pmb{y}_2\|_2^2)}
\\ &=
\lambda_2(\pmb{M}_{J,J}-\rho \hat{\pmb{z}} \hat{\pmb{z}}^\top)
+ (\lambda_1( \pmb{M}_{J^c,J^c}-\rho \hat{\pmb{Z}}_{J^c,J^c} ) - \lambda_2(\pmb{M}_{J,J}-\rho \hat{\pmb{z}} \hat{\pmb{z}}^\top)) \cdot t
+ 2 \| \pmb{M}_{J^c,J}-\rho \hat{\pmb{Z}}_{J^c,J} \|_2 \cdot 
\sqrt{t \cdot(1-t)}
\end{align*}
where $0\leq t := \|\pmb{y}_2\|_2^2 \leq 1$.
The first inequality holds since $\pmb{y}_1/\|\pmb{y}_1\|_2$ is orthogonal to $\hat{\pmb{x}}$, the leading eigenvector of $\pmb{M}_{J,J}-\rho \hat{\pmb{z}} \hat{\pmb{z}}^\top$.
The above upper bound of $\pmb{y}^\top (\pmb{M}-\rho \hat{\pmb{Z}}) \pmb{y}$ implies that if the following inequality holds for any $t\in [0,1]$:
\begin{multline*}
\lambda_2(\pmb{M}_{J,J}-\rho \hat{\pmb{z}} \hat{\pmb{z}}^\top)
+ (\lambda_1( \pmb{M}_{J^c,J^c}-\rho \hat{\pmb{Z}}_{J^c,J^c} ) - \lambda_2(\pmb{M}_{J,J}-\rho \hat{\pmb{z}} \hat{\pmb{z}}^\top)) \cdot t
+ 2 \| \pmb{M}_{J^c,J}-\rho \hat{\pmb{Z}}_{J^c,J} \|_2 \cdot 
\sqrt{t \cdot(1-t)}
\\ \leq
\lambda_1(\pmb{M}_{J,J}-\rho \hat{\pmb{z}} \hat{\pmb{z}}^\top),
\end{multline*}
then $\lambda_1(\pmb{M}_{J,J}-\rho \hat{\pmb{z}} \hat{\pmb{z}}^\top)$ is the largest eigenvalue of $\pmb{M}-\rho \hat{\pmb{Z}}$.
From Lemma \ref{lemma_quad}, we have that if the following inequality holds:
$$
\| \pmb{M}_{J^c,J}-\rho \hat{\pmb{Z}}_{J^c,J} \|_2^2 
\leq 
\big\{ \lambda_1(\pmb{M}_{J,J}-\rho \hat{\pmb{z}} \hat{\pmb{z}}^\top) - \lambda_2(\pmb{M}_{J,J}-\rho \hat{\pmb{z}} \hat{\pmb{z}}^\top) \big\} 
\cdot \big\{\lambda_1(\pmb{M}_{J,J}-\rho \hat{\pmb{z}} \hat{\pmb{z}}^\top) 
- \lambda_1( \pmb{M}_{J^c,J^c}-\rho \hat{\pmb{Z}}_{J^c,J^c} ) \big\},
$$
then
$\lambda_1(\pmb{M}_{J,J}-\rho \hat{\pmb{z}} \hat{\pmb{z}}^\top) = \lambda_1(\pmb{M}-\rho \hat{\pmb{Z}})$.
\end{proof}

\begin{lemma}
\label{lemma_quad}
Assume $a\neq 0$. If $a^2 \leq c(b+c)$ holds,
then $2a\sqrt{t(1-t)} \leq bt + c \text{ for all } t\in[0,1]$.
\end{lemma}

\begin{proof}
\begin{align*}
&2a\sqrt{t(1-t)} \leq bt + c ~~~\text{ for all } t\in[0,1]
\\ \Leftarrow&~
4a^2 t(1-t) \leq (bt + c)^2,~ bt + c \geq 0 ~~~\text{ for all } t\in[0,1]
\\ \Leftrightarrow&~
(4a^2 + b^2)\bigg(t - \frac{2a^2 - bc}{4a^2 + b^2}\bigg)^2 + c^2 - \frac{(2a^2 - bc)^2}{4a^2 + b^2} \geq 0,~ bt + c \geq 0 ~~~\text{ for all } t\in[0,1] 
\\ \Leftarrow&~
c^2 - \frac{(2a^2 - bc)^2}{4a^2 + b^2} \geq 0,~ c \geq 0,~ b+c \geq 0
\\ \Leftrightarrow&~
a^2 \leq c(b+c).
\end{align*}
\end{proof}

\end{document}